%% file: m468.tex
\documentclass{ecai} 
\usepackage{adjustbox}
\usepackage{latexsym}
\usepackage{amssymb}
\usepackage{amsmath}
\usepackage{amsthm}
\usepackage{multirow}
\usepackage{booktabs}
\usepackage{makecell}
\usepackage{enumitem}
\usepackage{graphicx}
\usepackage{color}
\usepackage{xcolor}
\usepackage{graphicx}
\usepackage{booktabs}
\usepackage{hyperref}
\usepackage{subcaption}
\usepackage{lipsum}
\usepackage{enumitem}
\usepackage{subcaption}
\usepackage{times}
\usepackage{algorithm}
\usepackage{enumitem}
\usepackage{tikz}
\usepackage{makecell}
\usepackage{circuitikz}
\usepackage[noend]{algpseudocode}

\newtheorem{theorem}{Theorem}
\newtheorem{lemma}[theorem]{Lemma}

\DeclareGraphicsExtensions{.eps,.pdf,.jpeg,.png}
\usetikzlibrary{ arrows.meta, positioning, fit, calc, backgrounds}
\tikzset{
  block/.style = {draw, rectangle, minimum height=2em, minimum width=3em, align=center},
  line/.style = {draw, -{Latex[length=2mm]}}
}
\usepackage{array}
\input{macros}
\renewcommand\paragraph{\@startsection{paragraph}{4}{\z@}%
                     {-12\p@ \@plus -4\p@ \@minus -4\p@}%
                     {-0.5em \@plus -0.22em \@minus -0.1em}%
                     {\normalfont\normalsize\bfseries}}
\makeatother

\begin{document}

%%%%%%%%%%%%%%%%%%%%%%%%%%%%%%%%%%%%%%%%%%%%%%%%%%%%%%%%%%%%%%%%%%%%%%%%

\begin{frontmatter}

%%% Use this command to specify your submission number.
%%% In doubleblind mode, it will be printed on the first page.

\paperid{468} 

%%% Use this command to specify the title of your paper.

\title{Survival of the Fittest: Evolutionary Adaptation of Policies for Environmental Shifts}

\author[A]{~\snm{Sheryl Paul}\thanks{Corresponding Author. Email: sherylpa@usc.edu}}

\author[A]{~\snm{Jyotirmoy V. Deshmukh}}

\address[A]{University of Southern California}

%%% Use this environment to include an abstract of your paper.

\begin{abstract}
Reinforcement learning (RL) has been successfully applied to solve the problem
of finding obstacle-free paths for autonomous agents operating in stochastic and
uncertain environments. However, when the underlying stochastic dynamics of the
environment experiences drastic distribution shifts, the optimal policy obtained
in the trained environment may be sub-optimal or may entirely fail in helping
find goal-reaching paths for the agent. Approaches like domain randomization and
robust RL can provide robust policies, but typically assume minor (bounded)
distribution shifts. For substantial distribution shifts, retraining (either
with a warm-start policy or from scratch) is an alternative approach. In this
paper, we develop a novel approach called {\em Evolutionary Robust Policy
Optimization} (ERPO), an adaptive re-training algorithm inspired by evolutionary
game theory (EGT). ERPO learns an optimal policy for the shifted environment
iteratively using a temperature parameter that controls the trade off between
exploration and adherence to the old optimal policy. The policy update itself is
an instantiation of the replicator dynamics used in EGT. We show that under
fairly common sparsity assumptions on rewards in such environments, ERPO
converges to the optimal policy in the shifted environment. We empirically
demonstrate that for path finding tasks in a number of environments, ERPO
outperforms several popular RL and deep RL algorithms (PPO, A3C, DQN) in many
scenarios and popular environments. This includes scenarios where the RL
algorithms are allowed to train from scratch in the new environment, when they
are retrained on the new environment, or when they are used in conjunction with
domain randomization. ERPO shows faster policy adaptation, higher average
rewards, and reduced computational costs in policy adaptation. 

% This achievement emphasizes its
% potential as a scalable solution for autonomous agents facing unpredictable
% changes in environments and the viability of evolutionary game theory principles
% in enhancing autonomous agents' adaptability and resilience.

% efficiently tailors control policies to shifted
% conditions, ensuring sustained optimization amidst large environmental changes.
% Theoretically, we present ERPO’s convergence to optimal policies in environments
% with sparse rewards; empirically we demonstrate that our algorithm outperforms
% state of the art Deep RL algorithms in shifted environments in a number of
% scenarios.
% 

%this assumption renders pre-trained policies suboptimal, necessitating adaptive
%re-training. While techniques like domain randomization and robust learning
%offer solutions against minor distributional variations, substantial
%environmental shifts demand a more rigorous approach for policy retraining to
%attain optimality. 
% We address this challenge with `Evolutionary Robust Policy
% Optimization' (ERPO), an algorithm inspired by evolutionary game theory for
%dynamic policy adaptation. 

\end{abstract}

\end{frontmatter}
\section{Introduction}
\input{intro}
\section{Preliminaries}
\input{preliminaries}
\vspace{-15pt}
\section{Solution Approach}
\input{solution}

\section{Experiments}
\input{experiments}

\section{Results}
\input{results}

\vspace{-5pt}
\section{Discussion}
\vspace{-5pt}
\mypara{Limitations and Future Work:}
Our set up is limited to discrete state-action spaces. We are working on an extension that works with continuous spaces. This will be carried out with function approximation using radial basis functions that also update the policies of states within a certain distance of the state we are updating. Additionally, because we normalize the probability distribution across actions of a given state, a continuous model would work instead along with a probability density function that can be updated using Dirac delta functions. Our set up is also limited to single agent models (unless extended with independent learning). We are working on extensions that can combine other game-theoretic solution concepts for cooperative multi-agent learning.

\mypara{Conclusion:} This paper presents a new approach to incrementally adapt the optimal policy of an autonomous agent in an environment that experiences large distirbution shifts in the environment dynamics. Our algorithm uses principles from evolutionary game theory (EGT) to adapt the policy and our policy update can be viewed as a version of replicator dynamics used in EGT. We provide theoretical convergence guarantees for our algorithm and empirically demonstrate that it outperforms several popular RL algorithms, both when the algorithms are warm-started with the old optimal policy, and when they are 
re-trained from scratch.

\input{m468.bbl}
\end{document}

%% file: macros.tex
\usepackage{amsmath,amssymb}

\DeclareMathOperator*{\argmax}{arg\,max}

\newcommand{\transitions}{\Delta}

\newcommand{\trainingpolicy}{{\pi}_{train}}

\newcommand{\agidx}{i}

\newcommand*{\agentstate}{s}
\newcommand*{\agentstates}{S}
\newcommand{\actions}{A}
\newcommand{\action}{a}
\newcommand{\rewards}{R}
\newcommand{\reward}{r}

\newcommand{\timeid}{t}

\newcommand{\utility}{{\utility}^{\agidx}_{\timeid}}

\newcommand{\totaltime}{T}
\newcommand{\return}{G}
\newcommand{\policy}{\pi}

\newcommand{\goalset}{F}

\newcommand{\traj}{\tau}

\newcommand{\optpolicy}{\policy^\star}
\newcommand{\newoptpolicy}{\policy^\star_{\Delta_{new}}}

\newcommand{\newtransitions}{\Delta_{new}}
\newcommand{\mypara}[1]{\vspace{0.2em}\noindent {\bf #1}}

%% file: intro.tex
A significant challenge for autonomous robotic agents used in automated
warehouses, autonomous driving, and multi-UAV missions is the problem of
identifying the optimal motion policy, i.e., for each state in the environment,
deciding the action that the agent should execute. There are several
computationally efficient approaches for planning the agent's actions in
deterministic and stochastic environments, especially when a model of the
environment is available \citep{uav,agmon_patrol, sven_warehouse,config,
sven_amazon, sven_flatland,sven_airport}. However, such models may not be
available for agents deployed in highly uncertain and dynamic environments \cite{usv, ship}.
Model-free reinforcement learning (RL) algorithms \citep{sutton,bertsekas} have
been highly effective at learning optimal policies when the environment dynamics
are unknown.
% learning policies in environments with highly complex state-action spaces,
% enabling applications in robot control \citep{ccdrl, rainbow}.

Traditional RL methods suffer from their lack of generalizability when
exposed to new, unanticipated changes in the environment. Typically, these RL
approaches demonstrate only moderate resistance to noise and exhibit poor
performance when deployed in environments significantly different from those
encountered during training. The lack of robust adaptation capabilities in these
algorithms is a critical drawback, especially in applications where reliability
and consistency across varied operational conditions are paramount.
\begin{figure}[ht!]
  \centering
  \begin{minipage}[b]{0.35\linewidth}
    \includegraphics[width=\linewidth, trim={0 1cm 1cm 0cm}, clip]{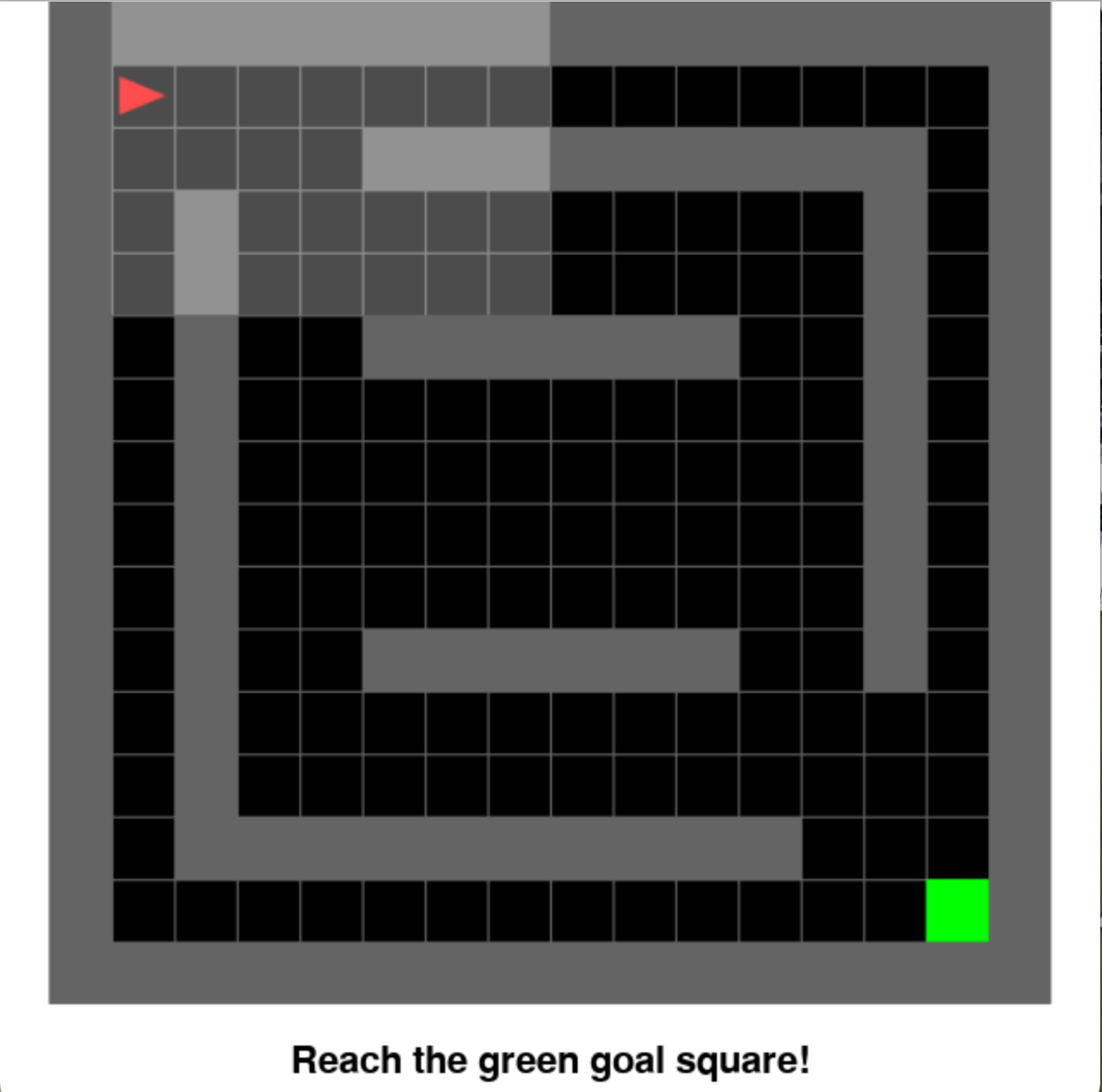}
    \caption*{Original Environment}
  \end{minipage}
  % \hfill
  \begin{minipage}[b]{0.35\linewidth}
    \includegraphics[width=\linewidth, trim={0 1cm 1cm 0cm}, clip]{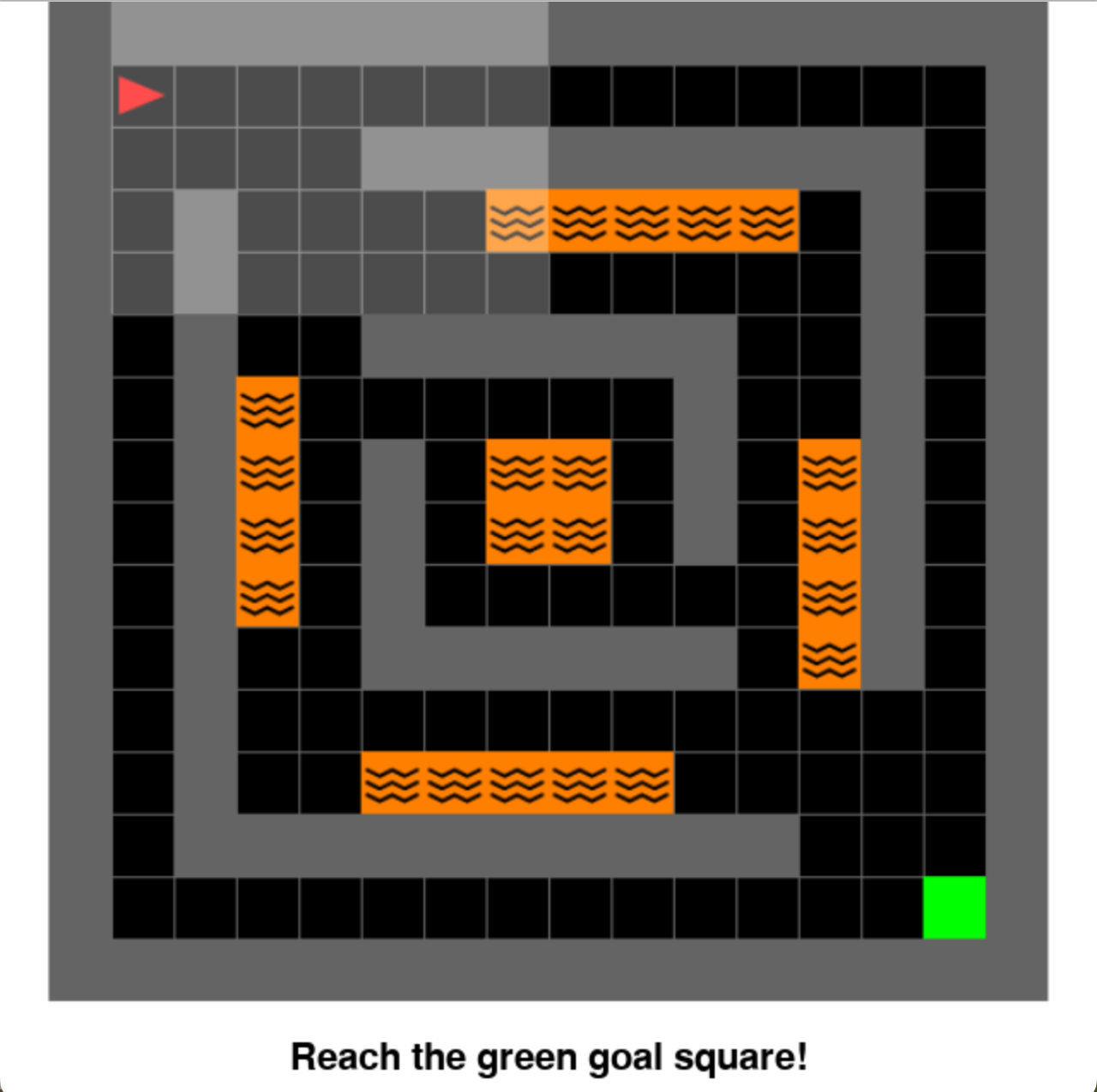}
    \caption*{Post distribution shift}
  \end{minipage}
  \vspace{5mm}
  \caption{The left figure is an example of an original environment where the agent (red triangle) has to reach the green goal square. The right figure represents the same environment after a distribution shift, introducing additional walls and `lava' that complicate navigation and makes it differ significantly from the original layout.}
  \label{fig:env-shift}
\end{figure}
\vspace{5mm}
\\
\mypara{Related Work:} In response to these challenges, several techniques have been developed to enhance the robustness of RL algorithms including domain randomization \cite{ppodr} and distributionally robust reinforcement learning (\citep{ distrrl} \citep{rarl}). Domain randomization  trains models across a wide range of simulated variations, thereby improving the algorithm's immunity to noise and its performance under environmental changes. However, this method generally does not perform well if the changes to the environment are substantial.

Adversarial RL \cite{arrl} and robust reinforcement learning techniques, including approaches like Monotonic Robust Policy Optimization (MRPO), specifically aim to optimize the algorithm’s performance in the worst-case scenarios. These approaches involve training under conditions that include adversarial disturbances or significant noise, thereby preparing the model to handle extreme situations \cite{optadversary,mrpo}. Although these methods significantly enhance the model's resilience, they sometimes fail to provide optimal solutions in less challenging or more typical scenarios indicating a trade-off between general robustness and peak performance

Current approaches in Robust RL (\citep{rarl}) focus on enabling model adaptation
to bridge the gap between simulation and real-world applications. Simulation models 
are generally simplistic and fail to consider environmental variables such as resistance, friction, and various other minor disturbances, so they cannot be directly deployed in risk-averse applications \cite{kadian2020sim2real}.

% There is also work in adversarial RL (\citep{optadversary}) that focuses on optimizing the extent of noise in the environment to train the model to be more robust in the worst case, or even training in the presence of adversarial actions (\citep{arrl}).
There is also theoretical work in developing versions of DQN such as DQN-Uncertain Robust Bellman Equation (\citep{urbe}) that focuses on developing Robust Markov Decision Processes (RMDPs) with a Bayesian approach.
Moreover, control theory-inspired approaches train models on subsets of underperforming trajectories. These methods focus on developing policies that exhibit greater durability and resilience in adverse conditions, often by considering worst-case performance guarantees as essential benchmarks for reliability \cite{epopt}, but once again, they suffer from being sub-optimal in many cases.

Approaches in transfer learning with deep reinforcement learning (Deep RL) \cite{transferrl} focuses on leveraging knowledge from previously learned tasks to accelerate learning in new, but related, environments. While this approach promises to improve adaptability and reduce training time, its  shortcomings include difficulties in identifying which parts of knowledge are transferable and the tendency to overfit to source tasks.

So we see that despite these advancements, there remains a substantial gap in effectively adapting pre-trained models to environments that undergo significant and sudden changes. Traditional RL methods are often ill-equipped to handle situations such as alterations in factory floor layouts, unexpected blockages in warehouse paths, or disruptions in road networks due to natural disasters or construction. These scenarios can drastically alter the dynamics of the environment, rendering previous optimal actions ineffective or suboptimal, and thereby demanding either complete retraining of the models or significant adjustments to their parameters and training protocols.

\mypara{Contributions.} To overcome prevalent limitations in traditional reinforcement learning, we introduce a novel method that synergizes RL-based planning with principles from evolutionary game theory. We generate batches of trajectories in a simulated perturbed environment and strategically explore this environment by employing an weighted version of the policy optimal in the original setting, enhancing adaptability. Subsequently, we refine the policy by prioritizing state-action pairs that demonstrate high fitness or returns, drawing on the concept of \textit{replicator dynamics} \citep{mapfegt}. This evolutionary game theory concept has been successfully applied in analyzing both normative and descriptive behaviors among agents \citep{tuyls, wooldridge}. Unlike traditional methods, our approach incrementally modifies the policy with new batches of data without relying on gradient calculations, ensuring { \em convergence to optimality with theoretical guarantees.}

We evaluate our algorithm `Evolutionary Robust Policy Optimization' (ERPO) based on this policy update  against leading deep reinforcement learning methods, including Proximal Policy Optimization (PPO) \citep{ppo}, PPO with Domain Randomization (PPO-DR) \citep{ppodr}, Deep Q-Network (DQN) \citep{dqn}, and Advantage Actor Critic (A2C) \citep{a2c}, both trained from scratch and retrained from a baseline model trained on the original environment environment (denoted as PPO-B, DQN-B, and A2C-B respectively). Our method demonstrates superior performance in various standard gym environments typical in RL research. Focused on discrete state and action spaces, we have applied our model to complex versions of environments such as \cite{towers_gymnasium_2023} \textit{FrozenLake}, \textit{Taxi}, \textit{CliffWalking}, \textit{Minigrid: DistributionShift}, and a challenging \textit{Minigrid} setup featuring walls and lava (\textit{Walls\&Lava}). Our findings reveal that our algorithm not only reduces computation times but also decreases the number of training episodes required to reach performance levels comparable to or better than those achieved by the aforementioned mainstream methodologies.

% \begin{table*}[htbp]
% \centering

% \label{tab:my_label}
% \begin{tabular}{>{\raggedright\arraybackslash}p{0.20\linewidth} >{\raggedright\arraybackslash}p{0.13\linewidth} >{\raggedright\arraybackslash}p{0.13\linewidth} >{\raggedright\arraybackslash}p{0.13\linewidth} >{\raggedright\arraybackslash}p{0.13\linewidth} >{\raggedright\arraybackslash}p{0.13\linewidth}}
% \toprule
% \textbf{Approach} & \textbf{Original/Adaptive} & \textbf{Noise Immunity} & \textbf{Performance in Significantly Changed Environment} & \textbf{Optimality Guarantee} & \textbf{Adaptation to Different Environment} \\
% \midrule
% Deep RL Algorithms (PPO/A2C/DQN) & Original & Medium & Low & Sometimes & Poor \\
% Domain Randomization & Adaptive & High & Medium & Worst-case  & Poor \\
% Adversarial RL & Adaptive & High & High & Worst-case  & No \\
% MRPO & Adaptive & High & High & Worst-case  & Yes \\
% \textbf{Our Work} & \textbf{Adaptive} & \textbf{Medium} & \textbf{High} & \textbf{Yes} & \textbf{Yes} \\
% \bottomrule
% \end{tabular}
% \caption{Comparison of Approaches. (Worst-case indicates that the algorithm may provide guarantees of optimality of performance in the worst case scenario but that they might be sub-optimal in other cases) }
% \end{table*}

%% file: preliminaries.tex
\subsection{Reinforcement Learning}
We model the system consisting of an autonomous agent interacting 
with its environment as a Markov Decision Process (MDP) defined as the 
tuple:
($\agentstates, \actions, \rewards, \transitions,
\gamma $). At each time-step $\timeid$, we assume
that the agent is in some state $\agentstate_\timeid \in \agentstates$, executes an action 
${\action}_{\timeid} \in \actions$, transitioning to the next state ${\agentstate}_{\timeid+1} \in \agentstates$,
and receiving a reward $\reward_{\timeid} \in \rewards$ with a discount factor $\gamma \in (0, 1]$. The transition dynamics
$\transitions$ is a probability of observing the next state ${\agentstate}_{\timeid+1}$ and getting the reward $\reward_{\timeid}$, given that the agent is in state ${\agentstate}_{\timeid}$ and takes the action ${\action}_{\timeid}$.

% \initialstates. \goalset, \obstacles

%We assume that we are provided with a set of initial %states 
%$\initialstates \subseteq \agentstates$. At time $0$, we assume that an 
%agent is assigned an initial state sampled randomly from $\initialstates$, i.e., \( {\agentstate}_{0}  \sim \initialstates  \).

% In RL, the purpose is to obtain the optimal {\em policy}, a map of the behaviors of an agent, which would lead to the maximum cumulative reward.
In this paper, we consider finite time-horizon (denoted as $T$) problems under a {\em stochastic} policy $\policy$, a probability distribution over actions given states, such that the action $\action_\timeid$
is sampled from the distribution $\policy(\action\mid \agentstate = \agentstate_\timeid)$ at any time $\timeid$.
We define $\goalset \subseteq \agentstates $ as the set of goal states in which the agent's task is considered to be achieved.
% i.e.  $s_t \in \goalset, t \leq T$.
Now we formalize the sparse reward setting: if $r(s_t, a_t, s_{t+1})$ is the reward received after taking action $a_t$ in state $s_t$. 
\[
r(s_t, a_t, s'_{t+1}) \gg r(s_t, a_t, s_{t+1}).
\]
 \ $\text{where} \ s'_{t+1} \in \goalset, \ \forall s_{t+1} \notin \goalset$.

A trajectory $\traj$ of the agent induced
by policy $\policy$ is defined as a $(\totaltime+1)$-length sequence of 
state-action pairs:
\begin{equation}
\label{eq:traj}
\begin{gathered}
    {\traj} = \{ 
    ({\agentstate}_0,{\action}_0 ), 
    ({\agentstate}_1,{\action}_1 ), 
    \ldots,
    ({\agentstate}_{\totaltime-1},{\action}_{\totaltime-1} ), {\agentstate}_{\totaltime} \}, \\
    \text{where, }
    \forall \timeid<\totaltime: \action_\timeid \sim \policy(\action \mid \agentstate = \agentstate_\timeid), 
    (s_{t+1}) \sim \Delta(s_{t+1} \mid s_t, a_t).
\end{gathered}
\end{equation}
We also donate a trajectory $\traj$ where actions have been sampled using the policy $\pi$ as $\traj \sim \pi$.
Given a trajectory, the total discounted reward of a trajectory is:
\[
{\return^\policy}({\traj}) = 
    \sum_{t=0}^{\totaltime}
    \gamma^t \reward_{t}
\] 
We define the state-value and action-value functions as follows:
\[
v^\pi(s) = \mathbb{E}_{\tau \sim \pi} \left[G^\pi(\tau) | s_0 = s\right]
\]
and
\[
q^\pi(s, a) = \mathbb{E}_{\tau \sim \pi} \left[G^\pi(\tau) | s_0 = s, a_0 = a\right]
\]
Let $\eta (\policy)$ be the expected discounted return for an agent  under the policy $\policy$ across all trajectories
\[
\eta(\pi) = \mathbb{E}_{\tau \sim \pi}\left[\sum_{t=0}^{T-1} \gamma^t r_{t+1}\right]
\]
The optimal policy $ \optpolicy $ for the MDP can then be defined as: 
\[
\optpolicy = \argmax \limits_{\policy} \ \eta (\policy),
\]
\vspace{-15pt}
\subsection{Evolutionary Game Theory}
{\em EGT} originated as the application of game-theoretic concepts to biological settings. This concept stems from the understanding that frequency-dependent fitness introduces a strategic dimension to the process of evolution \citep{stanforddef}.
It models Darwinian competition and can be a dynamic alternative to traditional game theory that also obviates the need for assumptions of rationality from the participating members. It has been applied to modeling complex adaptive systems where strategies evolve over time as those yielding higher payoffs become more prevalent, akin to the survival of the fittest in natural selection. We aim to leverage the principles of EGT \citep{egt_games}, \citep{egt_sandholm} to build an approach for our distribution shift problem.
% \mypara{Application Context:} In the context of RL, this can be applied by considering each policy as a member of a population, and the environment serves as the game setting. When we map agent policies onto a population, we treat the probability distribution over actions (i.e., the policy) as a strategy within the evolutionary game. This policy's success is measured by its returns, analogous to an organism's fitness in biological evolution. Policies that lead to higher returns are akin to having higher fitness—will have their probabilities "replicated" or increased, ensuring that successful behaviors are more likely to be evident in future interactions with the environment.

\subsubsection{Replicator Dynamics Equation:}
The key concept within EGT pertinent to us is that of {\em replicator dynamics}, which describes how the frequency of strategies (or policies in RL) changes over time based on their relative performance.
The classic replicator equation in evolutionary game theory describes how the proportion of a population adopting a certain strategy evolves over time. Mathematically, it is expressed as \citep{replicator}:
\begin{equation}
    x_j(i+1) = x_j(i) \frac{f_j(i)}{\bar{f}(i)} \label{eq:og_rep}
\end{equation}
where $x_j(i)$ represents the proportion of the population using strategy $j$ at time $i$, $f_j(i)$ is the fitness of strategy $j$, and $\bar{f}(i)$ is the average fitness of all strategies at time $i$. The equation indicates that the growth rate of a strategy's proportion is proportional to how much its fitness exceeds the average fitness, leading to an increase in the frequency of strategies that perform better than average.
\newline

\mypara{Problem Definition:}
% We assume that we have an optimal policy  $\optpolicy$ (we will simply refer it as $\optpolicy_{\transitions}$ or $\optpolicy_{old}$) for the original environment with dynamics $\transitions$. 
We assume that we have an optimal policy for the original environment dynamics $\transitions$, referred to as \(\optpolicy_{\transitions}\).

Suppose that we have a new environment with dynamics $\newtransitions$ obtained by significantly perturbing the distribution representing \(\transitions\),\footnotemark{}%
then the problem we wish top solve is to learn a new policy \(\newoptpolicy\) such that
% The problem we wish to solve is learning a new policy $\newoptpolicy$, s.t.,
\begin{equation}
    \newoptpolicy\ = \ \argmax\limits_{\policy}\  \eta_{\newtransitions} (\policy).
\end{equation}
\footnotetext{Let $\beta \ = \ \ D_{TV}(\transitions || \transitions_{new})$, i.e. the total variation distance between the transition dynamics of the old and new environments; then $\beta$ is bounded as $\beta \leq \beta_{hi}$.
For our experiments, $\beta_{hi} = 0.4$.}

%% file: solution.tex
\subsection{Translation of the Replicator Equation}

\mypara{Representation of populations and the fitness equivalent:}  We represent the probability distribution over actions in a given state as a population, where each type of population corresponds to a possible action. For a system comprising $n$ states with $m$ possible actions in each state, we effectively have $n$ distinct populations of $m$ types each. This results in $m^n$ different types of individuals in the aggregated state population.

The fitness function measures the reproductive success of strategies based on payoffs from interactions, similar to utility in classical game theory, or how in RL, the expected return measures the long-term benefits of actions based on received rewards. Both serve as optimization criteria: strategies or policies are chosen to maximize these cumulative success measures to guide them towards optimal behavior.
Therefore, in our model, the fitness for a state \( f(s) \) corresponds to the expected return from that state, equivalent to the value function \( v(s) \), and the fitness for a state-action pair \( f(s,a) \) corresponds to the expected return from taking action \( a \) in state \( s \), equivalent to the action-value function \( q(s,a) \). 

Under the assumption of sparse rewards—where significant rewards are received only upon reaching specific states or goals—\( f(s) \) is defined as \( \mathbb{E}[f(\tau_s)] \), the expected return across all trajectories through state \( s \). Likewise, \( f(s,a) \) is defined as \( \mathbb{E}[f(\tau_{(s,a)})] \), the expected return across trajectories involving the state-action pair \( (s,a) \). 
\begin{align}
    q(s,a) &= f(s,a) \approx \mathbb{E}[f(\tau_{(s,a)})],\label{eq:q_sparse} \\
    v(s)   &= f(s) \approx \mathbb{E}[f(\tau_{(s)})] \label{eq:v_sparse}
\end{align}

\mypara{Policy Update Mechanism:} The replicator equation can be adapted to update the probability of selecting certain actions based on their relative performance compared to the average. The adaptation of the replicator equation is as follows:
\begin{equation}
\label{eq:policyupdate}
    \pi^{i+1} (s,a) = \frac{\pi^i (s,a) f(s,a)}{\sum_{a' \in A} \pi^i (s,a') f(s,a')} 
\end{equation}
where $\policy^{i}(\agentstate, \action) $ is the  probability of action $\action$ in state $s$, in the $i^{ith}$ iteration, and $\pi^{i+1}$ represents the policy int he ${i+1}^{th}$ iteration. 
% The policy update in Eq.~\eqref{eq:policyupdate} can be viewed as simultaneous application of the
%  replicator equation in Eq.~\eqref{eq:og_rep} to all sampled states.

% across all $n$ states can be envisioned as a stack, or a Cartesian product, of \( s \times a \) pairs of replicator equations, where each pair corresponds to a state-action combination.

\begin{lemma}
Policy update in Eq.~\eqref{eq:policyupdate} encodes the replicator dynamics in Eq.~\eqref{eq:og_rep}.
% \[
% \pi^{i+1} (s,a) = \pi^i (s,a) \times \frac{ E[f(s,a)]}{\sum_{a' \in A} \pi^i (s,a') E[f(s,a')]}
% \]
    
\end{lemma}

% \begin{proof}
% We delineate the translation of the replicator equation to our reinforcement learning dynamics as follows.
% The original replicator equation can be denoted as:
% \begin{equation}
% {x}_i(t+1) = x_i(t) \frac{f_i(t)}{\bar{f}(t)}
% \end{equation}

The proof of the above lemma follows from the observation that for a state $s$, and a specific action $a_i$, the replicator equation \ref{eq:og_rep} in our setting would look like\footnote{Assuming the fitness function is not time-dependent}:
\begin{equation}
    x_{s,a_j}(i+1) = x_{s,a_j}(i) \frac{f(s,a_j)}{\sum_{a \in A} f(s,a) \cdot x_{(s,a)}(i)}
    \label{eq:intermediate_policy}
\end{equation} 
By our representation of policies as populations, we see that $\pi^i(s,a_j)$ is equivalent to $x_{s,a_j}(i)$, and so Eq. (~\ref{eq:policyupdate})
follows from Eq. (~\ref{eq:intermediate_policy}).
 The policy update in Eq.~\eqref{eq:policyupdate} is just a simultaneous
application of parallel replicator equations to all states (being updated in a given iteration).

%     \pi^{i+1} (s,a) &= \frac{\pi^i (s,a) E[f(s,a)]}{\sum_{a' \in A} \pi^i (s,a') E[f(s,a')]} 
% \end{align*}
%  

This rule essentially captures the essence of the replicator dynamic by
adjusting the probability of action \( a \) in state \( s \) proportionally to
its performance relative to the average performance of all actions in that
state. The normalizing factor in the denominator ensures that the updated policy
remains a valid probability distribution, aligning with the principle of the
replicator dynamic where strategy frequencies within a population must sum to
one.\footnote{The original replicator equation describes the evolution of
strategy proportions within a population, not the absolute numbers of
individuals employing each strategy. This focus on proportions makes it a
suitable model for normalizing factors in our policy update equation. Therefore,
the replicator dynamic's mechanism, which adjusts strategy frequencies based on
relative fitness, is analogous to adjusting the probability of action selections
in relation to their expected return, thus ensuring that \( \pi(s,a) \) remains
a normalized probability distribution.}
% \end{proof}

% \subsection{Proof of convergence:}

\begin{theorem}
The algorithm (ERPO) that employs the policy update specified in Eq.~\eqref{eq:policyupdate},
% \[
%\pi^{i+1} (s,a) = \pi^i (s,a) \times \frac{ E[f(s,a)]}{\sum_{a' \in A} \pi^i (s,a') E[f(s,a')]}
% \]
ensures that the value of each state monotonically improves with each iteration, converging to an optimal policy under assumptions of sparse rewards.
\end{theorem}

\begin{proof}
We note that our algorithm employs a batched Monte Carlo-style sampling approach, collecting multiple trajectories in each batch. We assume that each batch is sufficiently large to ensure that the estimated values of the $v$ and $q$ functions closely approximate their true values, so that Eq. (\ref{eq:v_sparse}) and Eq. (\ref{eq:q_sparse}) hold. We also assume that each state is visited at least once in each batch.

We define the action-value function and value function under policy \(\pi^i\), accounting for transition probabilities:
\[
q^i(s,a) = \sum_{s' \in S} \Delta(s,a,s') \left(r(s,a,s') + \gamma v^i(s')\right),
\]
\begin{equation}
    v^i(s) = \sum_{a \in A} \pi^i(s,a) q^i(s,a).
    \label{eq:v_q}
\end{equation} 

We now partition the set of actions $A$ into $A_h$ and $A_l$ such that:
\[
A_h = \{a_h \in A \mid q^i(s,a_h) \geq v^i(s)\} 
\]
\[
A_l = \{a_h \in A \mid q^i(s,a_l) < v^i(s)\}.
\]
% The new value function \(v^{i+1}(s)\) is:
% \[
% v^{i+1}(s) = \sum_{a \in A} \pi^{i+1}(s,a) q^i(s,a),
% \]
Splitting into contributions from \( A_h \) and \( A_l \):
\begin{equation}
v^i(s) = \Sigma_{a_h \in A_h} \pi^i(s,a_h) q^i(s,a_h) + \Sigma_{a_l \in A_l} \pi^i(s,a_l) q^i(s,a_l)
\label{eq:v_k}
\end{equation}
Similarly,
\begin{equation}
 v^{i+1}(s) = \Sigma_{a_h \in A_h} \pi^{i+1}(s,a_h) q^i(s,a_h) + \Sigma_{a_l \in A_l} \pi^{i+1}(s,a_l) q^i(s,a_l)
\label{eq:v_k_plus_1}
\end{equation}
By our sparse reward assumption from Eqs. (\ref{eq:q_sparse}), (\ref{eq:v_sparse}) and  (\ref{eq:v_q}) we can state that:
\begin{equation}
    v^i(s)  = f(s) = \sum_{a' \in A} \pi^i (s,a') f(s,a')
\end{equation}
And the policy update equation can now be modified as:
\begin{equation}
    \pi^{i+1}(s,a) = \pi(s,a) \left[\frac{q^i(s,a)}{v^i(s)}\right] \label{eq:new_up}
\end{equation}
%  We additionally assume non-negative rewards, and that rewards are nearly zero except in a few significant states - goal states, such that:
% \begin{equation}
%     r(s,a,s')  
%     \begin{cases} 
%    \approx 0 & \text{if } s' \notin \goalset, \\
%     >>0 & \text{otherwise}.
%     \end{cases}
% \end{equation}

% \sheryl{I want to clarify something in the below eqn- I don't think it's very clear}
% $q^i(s,a_h) \geq v^i(s), \therefore $
By definition:  \( q^i(s,a_h) \geq v^i(s)\) and \(q^i(s,a_l) < v^i(s)\), and from Eq.  (\ref{eq:new_up}) we get:
\[
\pi^{i+1} (s,a_h) \geq \pi^{i} (s,a_h) \ \text{ ; }
\pi^{i+1} (s,a_l) < \pi^{i} (s,a_l)
\]
the updated policy increases the probability of selecting \(a_h\) and decreases the probability of selecting \(a_l\), so from  Eqs. (\ref{eq:v_k}) and (\ref{eq:v_k_plus_1}) we get:
\[
v^{i+1}(s) \geq v^i(s)
\]
Therefore, we show that a policy iteration algorithm based on this update rule guarantees that each state's value monotonically improves, ensuring convergence to the optimal policy.\footnote{
Convergence to an optimal policy in this setting with an evolutionary update is analogous to the  concept of convergence to an evolutionarily stable strategy (ESS). An ESS is a strategy that, if adopted by a population, cannot be invaded by any alternative strategy that is initially rare. This implies that in this setting, once an optimal policy is reached, it cannot be outperformed easily by any other policy (under standard ESS assumptions \cite{stanforddef}) thereby ensuring that the agent's behavior is robust against most changes and variations in the strategy space.}
\end{proof}

\noindent \emph{Remark.} Our algorithm operates in a Markovian framework, meaning state transitions depend only on the current state and action, without influence from past states/actions. Consequently, the replication of strategies and policy/value improvements can be applied to all states independently, where each state-action pair is updated without interference from the updates of the others. This facilitates parallel improvements across all states towards an optimal policy.
\subsection{Evolutionary Robust Policy Optimization (ERPO)}

% \begin{figure}[!ht]
% \centering
% \fbox{ % This will create a frame around the tikzpicture
% \begin{tikzpicture}[font=\scriptsize, node distance=1cm and 1cm]
%   % Nodes
%   \node [block] (pretrained) {Base\\Model};
%   \node [block, below=of pretrained] (newenv) {New Environment\\$(\Delta_{new})$};
%   \node [block, below=of newenv] (policyupdate) {Evolutionary Policy Update };
%   \node [block, right=2cm of pretrained] (trainpolicy) {Training Policy Update:\\$\pi_{train}^{i} = w.\pi^i_{\Delta} + (1-w).\pi_{new}^i$};

%   \begin{scope}[on background layer]
%   \node[draw=black, dashed, fit=(pretrained), inner sep=6pt, label=Original Environment $(\Delta)$] (box1) {};
%   \end{scope}

%   % Arrows
%   \draw [line] (pretrained) -- (newenv) node[midway, fill=white] {$\pi^*_{\Delta}$};
%   \draw [line] (newenv) -- (policyupdate) node[midway, fill=white] {Generate trajectory batch $\tau_B$};
%   \draw [line] (policyupdate.east) -| (trainpolicy.south) node[near start, fill=white] {$\pi_{new}^{i+1}$};

%   \draw [line] (trainpolicy.west) -- ++(-1,0) |- (newenv) node[near start, fill=white] {$\pi_{train}^i$};

% \end{tikzpicture}
% } % Closing fbox
% \caption{Training process with policy update}
% \vspace{5mm}
% \label{fig:training_process_with_policy_update}

% \end{figure}
\begin{figure}[!ht]
    \centering
    \includegraphics[trim=4cm 0cm 7cm 0cm, clip, width=0.5\textwidth]{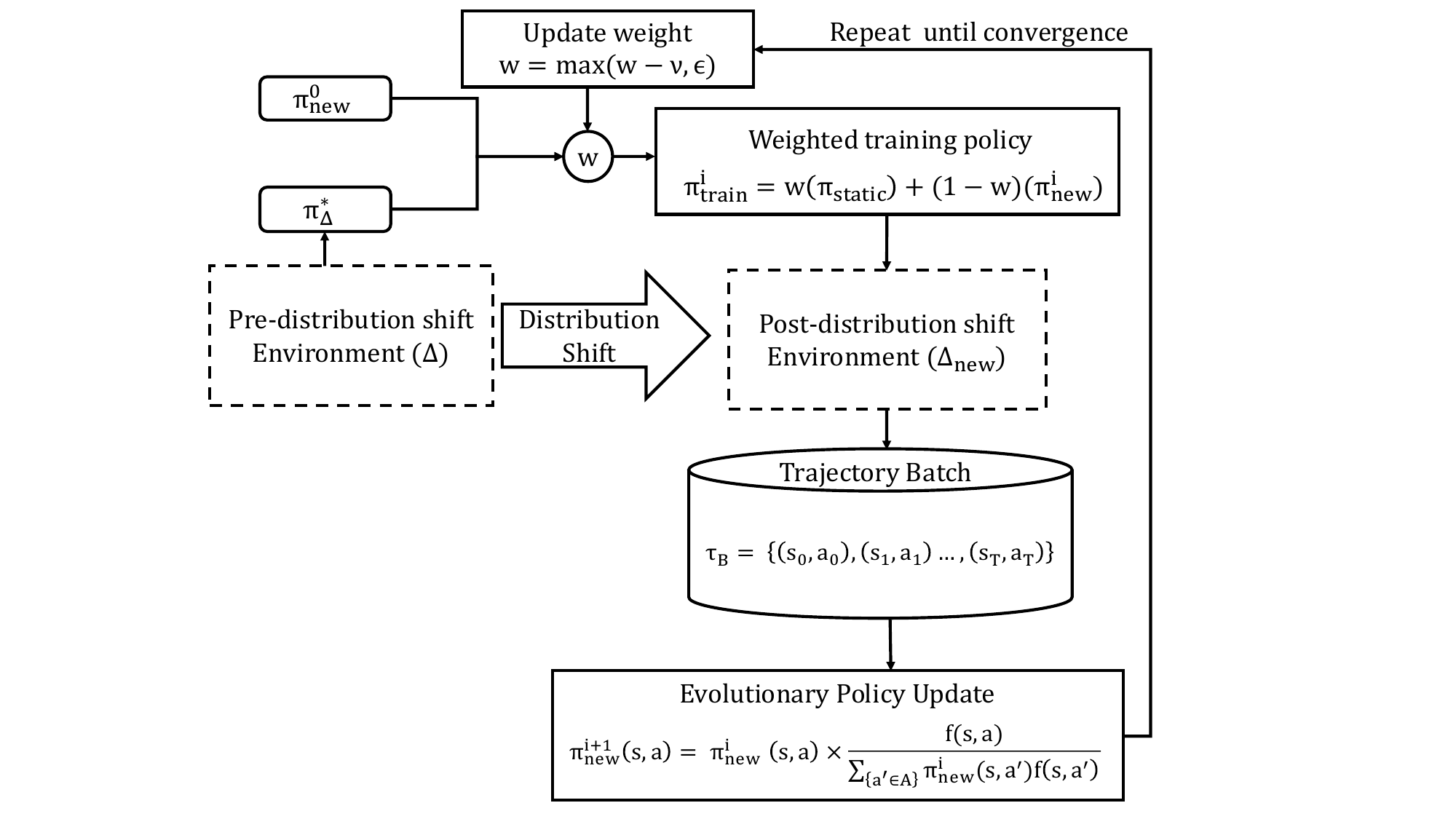}
    \caption{Outline of the ERPO methodology.}
    \vspace{15pt}

    \label{fig:flowchart}
\end{figure}

\begin{algorithm}
\caption{{\sc Evolutionary Robust Policy Optimization}}
\label{algo:erpo_algorithm}
\begin{algorithmic}[1]
\State \textbf{Input:}
\begin{itemize}[nosep, label={}]
    \item Optimal policy $\optpolicy_{\Delta} = \argmax_{\policy} \eta_\transitions (\policy)$
    \item Initialize $\forall \agentstate \in S, \action \in A: \pi^0_{new}(\agentstate, \action) = \frac{1}{|A|}$
    \item Hyperparameters $\epsilon, \nu \in (0,1), \delta > 0$
\end{itemize}
\State \textbf{Output:} Optimized policy $\optpolicy_{new}$
\State Initialize $i \gets 0$, $\eta^0$, $\trainingpolicy^0  \gets w \optpolicy_{\Delta} \ +(1- w)\pi^{i+1}_{new}$ \label{line:init_policy}
\Repeat
    \For{each episode in batch b = 1 to B}
        \State Generate trajectory $\traj_b \sim \pi_{train}$
        \State Append trajectory $\traj_b$ to batch
    \EndFor
    \For{all trajectories $\traj_b$ in batch}
        \For{ each $(s,a) \in \traj_b$}
        \State Update $\pi_{new}^{i+1} (s,a) = \frac{\pi_{new}^i (s,a) f(s,a)}{\sum_{a' \in A} \pi_{new}^i (s,a') f(s,a')}$ \label{line:optimize_policy}
        \EndFor
    \EndFor
    \State Update expected return: $\eta^{i+1} \gets \eta_{\transitions_{new}}(\trainingpolicy^i)$ \label{line:return_update}
    \State Update training policy $\trainingpolicy^{i+1} \gets w \optpolicy_{\Delta} \ +(1- w)\pi^{i+1}_{new}$ 
    \State Decrement $w \gets max( w - \nu, \epsilon)$; 
    \State Increment $i \gets i + 1$  \label{line:update_weight_counter}
    
\Until{Convergence criteria: $(\eta^{i+1} - \eta^i \leq \delta)\ )$}
\State \Return $\pi*_{new} \gets \trainingpolicy$
\end{algorithmic}
\end{algorithm}

 Our algorithm uses batch-based updates: We initialize our training policy to be a weighted combination of the old optimal policy $\optpolicy$, and our new policy $\policy_{new}$ -  which is initially random. We sample trajectories as part of a batch, and in doing so we make sampling assumptions as mentioned earlier, and the state-action pairs in these trajectories are updated according to the update rule. The return for the $i+1^{th}$ iteration is set as the return under the training policy, and the training policy is updated, to take into account the update to $\pi_{new}$. The weight assigned to the old policy is decremented with each iteration. This process is repeated until our termination condition $(\eta^{i+1} - \eta^i > \delta)$ is met. The termination condition checks if the expected return across all states is changing over our batch runs, and when the difference in the expected returns across consecutive batches is minimal, we say that the algorithm has converged.

%% file: experiments.tex
%  For our experiments we compare our method against PPO-DR \cite{ppodr} (i.e., PPO trained with domain randomization to improve its robustness), PPO,  DQN \cite{dqn}, A2C (Advantage Actor-Critic) \cite{a2c}, when re-trained from scratch on the new world parametrized by $\transitions_{new}$ (except for PPO-DR), as well as when they are trained over the model that produces ${\pi}^*$ in the world parametrized by $\transitions$. (These are the base models indicated in the figures, and the algorithms are indicated as A2C-B, PPO-B, and DQN-B, respectively.)  We have used the `stable-baselines3' \cite{stable-baselines3} implementation for the mentioned Deep RL algorithms and tuned the hyper-parameters using the `optuna' \cite{optuna} library.\footnote{The experiments were performed on a high-performance computing cluster with nodes using dual 8-16 core processors, using 16 CPUs with 32GB memory each.}
% \\
% \\
\mypara{Benchmarks:} In our empirical analysis, we benchmark our approach
against a selection of established reinforcement learning algorithms. The
comparison is conducted under two different scenarios: (1) each RL algorithm is
allowed to train on the modified environment from scratch, (2) we obtain
pre-trained corresponding to the optimal policy, and then train them over
the modified environment\footnote{We use the term `model' to describe the specific
neural network(s) used in each RL algorithm; for example, for PPO, this means the
policy network.}. Each baseline scenario is described in detail below:

\begin{itemize}
  
  \item \textbf{PPO} \cite{ppo}: Standard Proximal Policy Optimization, re-trained from scratch in the new environment.
  
  \item \textbf{DQN} \cite{dqn}: Deep Q-Network, also re-trained from scratch in the new environment.
  
  \item \textbf{A2C} \cite{a2c}: Advantage Actor-Critic, re-trained from scratch in the new environment.
  
  \item \textbf{PPO-B, DQN-B, A2C-B}: These baselines correspond to models that are trained by warm-starting the
  training in the new environment with the old optimal policy.
  
  \item \textbf{PPO-DR} \cite{ppodr}: Proximal Policy Optimization with Domain Randomization, enhances robustness by training across varied environments.
\end{itemize}

Implementations for these algorithms were sourced from the `stable-baselines3' library \cite{stable-baselines3}, with hyperparameter optimization facilitated by the `optuna' library \cite{optuna}. All experiments were conducted on a high-performance computing cluster\footnote{The computational resources included nodes with dual 8-16 core processors, 16 CPUs, and 32GB of memory per node.}.

\mypara{Environments:}
All the benchmarks are tested on the {\em FrozenLake}, {\em CliffWalking}, and {\em Taxi} environments in Open AI gymnasium,  {\em Minigrid's Distribution Shift} environment \cite{Minigrid} and a version of the {\em Minigrid: Empty} environment customized with walls and lava.
We remark that we use larger and more complex versions of the standard environments to test our algorithm properly. More precisely we vary the Total Variation Distance between $\Delta$ and $\Delta_{new}$ i.e. the transition dynamics of the old and new environments between 0.15 and 0.4.
\begin{enumerate}[nosep,wide, labelwidth=!, labelindent=0pt]
    \item {\em FrozenLake (\textbf{FL}):} The agent tries to navigate across a Frozen Lake to reach a goal. The episode terminates when the agent enters a hole and drowns, or reaches the goal. We have a base model with few holes and three additional levels with increasing grid environment occupied by holes, indicated with the darker blue in Fig.~\ref{fig:frozen_lake_envs}.
    \item {\em CliffWalking (\textbf{CW})}: The agent starts on the bottom left and must reach the goal location (bottom right) while avoiding  `cliff' locations (indicated in brown), otherwise it is returned to the start position. We have a base model with one row of cliffs (similar to the standard model), and three additional levels with increasing cliff area ( Fig.~\ref{fig:cw_envs}). 
    \item { \em Taxi (\textbf{TX})}: The agent (a taxi) must pick up and drop a passenger from designated stations (indicated in boxes of red, green, blue and yellow). Dividers prevent the taxi from turning left or right, forcing it to take a more circuitous route and make U-turns. The base model has no dividers, and additional levels have increasing numbers of dividers (see Fig. ~\ref{fig:taxi_envs}). 
    \item { \em Minigrid: DistributionShift (\textbf{MGDS})}: The purpose is to test the ability to generalize across two variations of the environment. The episode terminates when the agent reaches the goal or lava. We have three levels of environments with increasing grid areas occupied by lava. See Fig.~\ref{fig:ds_envs}.
    \item {\em Minigrid: Walls\&Lava (\textbf{MGWL})}: Lastly, we test the ability to navigate in the presence of walls that block the agent's vision and movement, or lava that terminates the episode, or both. The base model is an empty grid, Level 1 has lava, Level 2 has walls, and Level 3 has both. See Fig.~\ref{fig:mgwl_envs}.
\end{enumerate}

\mypara{Implementation Details:} We assume that the optimal policy in the
original environment is obtained using PPO (this can be changed to other
algorithm). The base PPO model is allowed learns over a minimum of $10^4$ and a
maximum of $10^6$ timesteps -- the number of steps required to converge to
(close to) $\pi^*_{\Delta}$ varies across environments. As we induce
distribution shifts, we need to pick reward functions that are sensible across
all instantiations of any environment. The reward functions for each environment
are standard across all the models of all the algorithms used. Further details
of the reward functions are presented in the results, and other details such as the hyperparameters can be found in the supplementary material. \footnote{\href{https://github.com/sherylpaul/ERPO}{https://github.com/sherylpaul/ERPO}}

\emph{Remark.}  We note that the reward function modification adheres to the
principles outlined in \cite{ngreward}, utilizing transformations of the reward
function that maintain policy invariance. Specifically, the rewards are
structured as potential-based transformations, where the modified reward
function is given by \( r'(s,a,s') = r(s,a,s') + \gamma \Phi(s') - \Phi(s) \),
with \(\Phi\) being a potential function. Despite incorporating scaled rewards,
these transformations preserve the optimality of the policies as the
potential-based adjustments ensure the fundamental characteristics of the
original reward system are maintained. This confirms the robust application of
our model and validates its efficacy across varied and complex reward
structures.
\begin{figure}[!ht]
\centering
    \begin{subfigure}{.24\columnwidth}
        \includegraphics[width=0.9\linewidth]{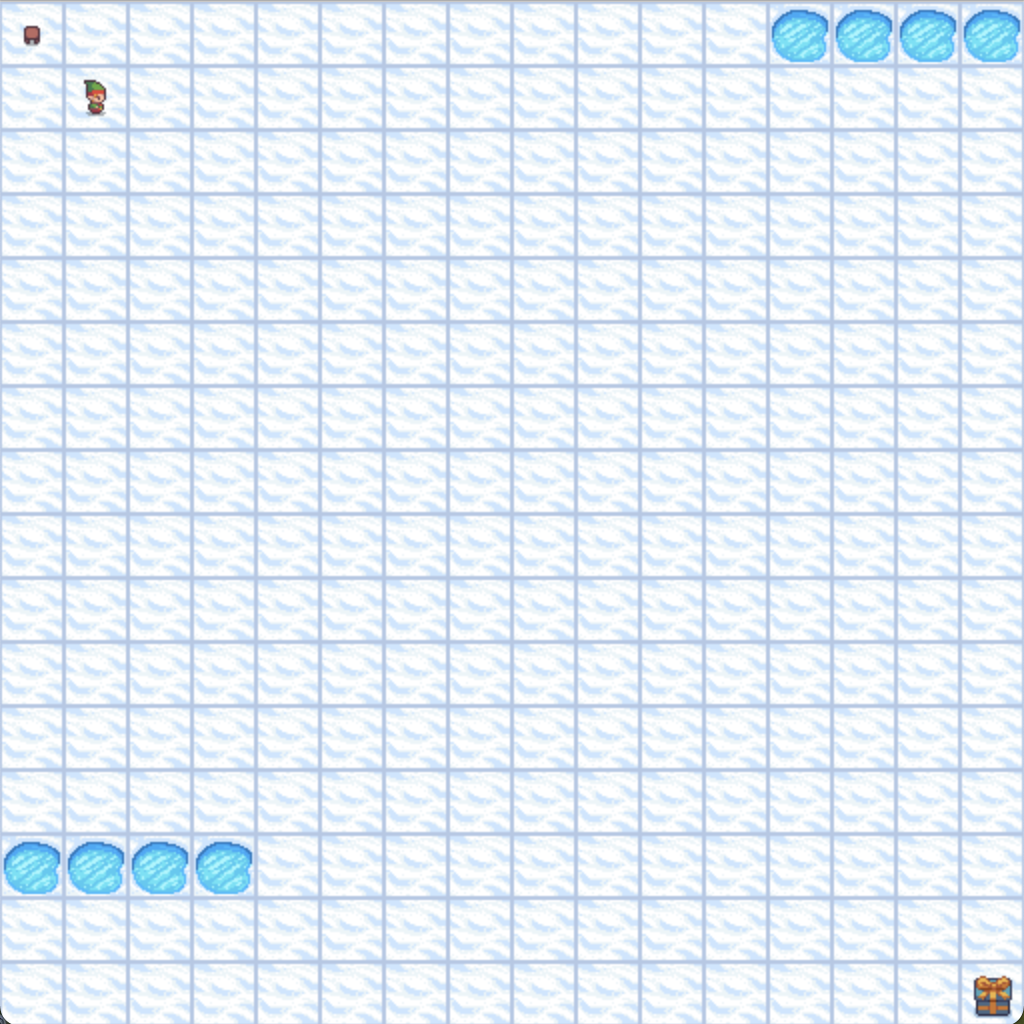}
        \caption{FL: Base}
        \label{fig:fl_s}
    \end{subfigure}%
    \begin{subfigure}{.24\columnwidth}
        \includegraphics[width=0.9\linewidth]{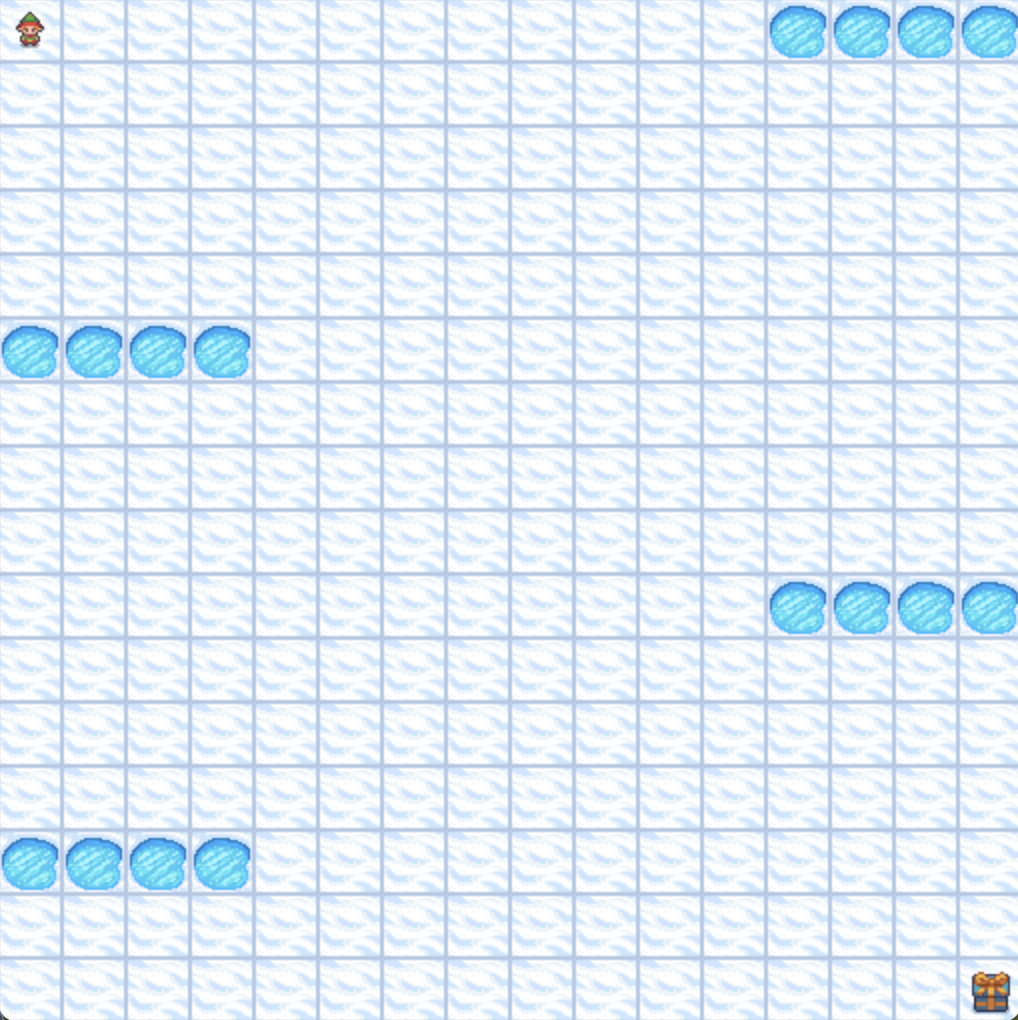}
        \caption{FL: L1}
        \label{fig:fl_l1}
    \end{subfigure}%
    \begin{subfigure}{.24\columnwidth}
        \includegraphics[width=0.9\linewidth]{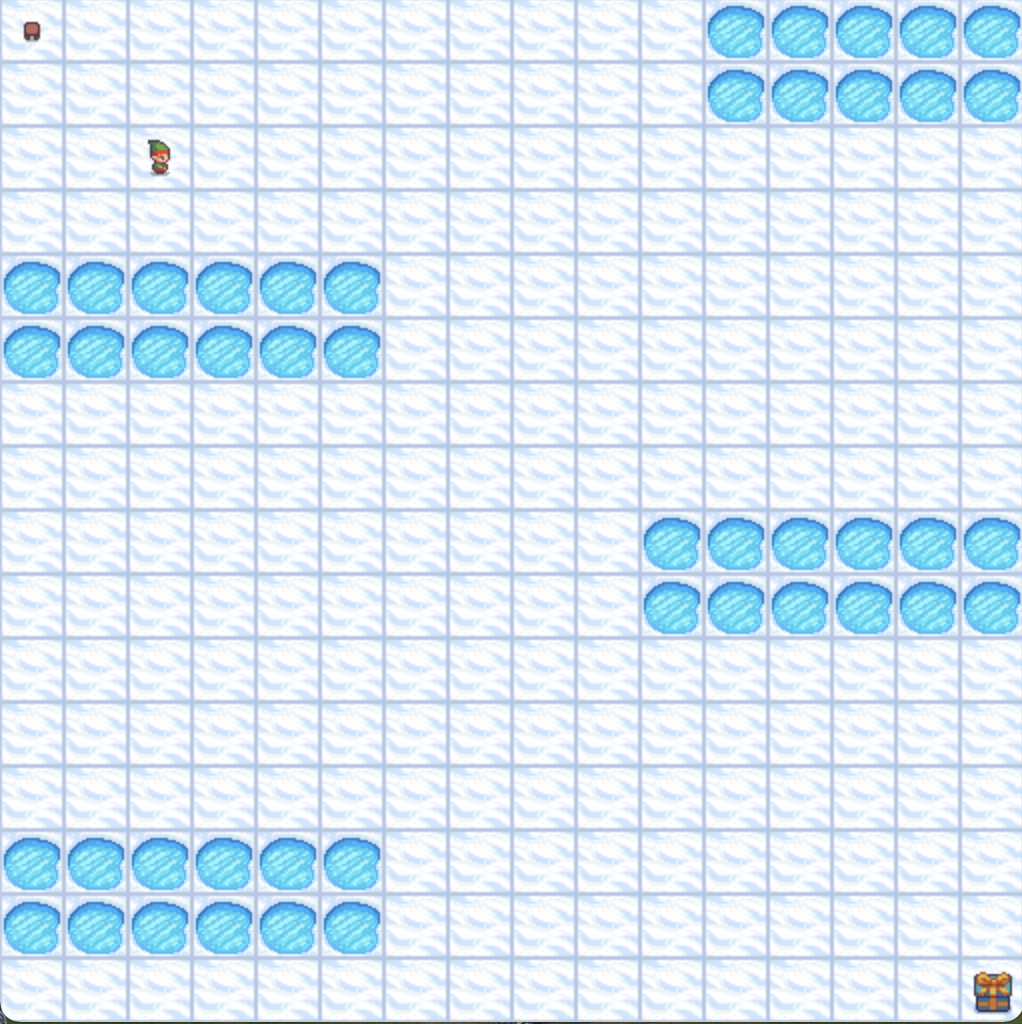}
        \caption{FL: L2}
        \label{fig:fl_l2}
    \end{subfigure}
    \begin{subfigure}{.24\columnwidth}
        \includegraphics[width=0.9\linewidth]{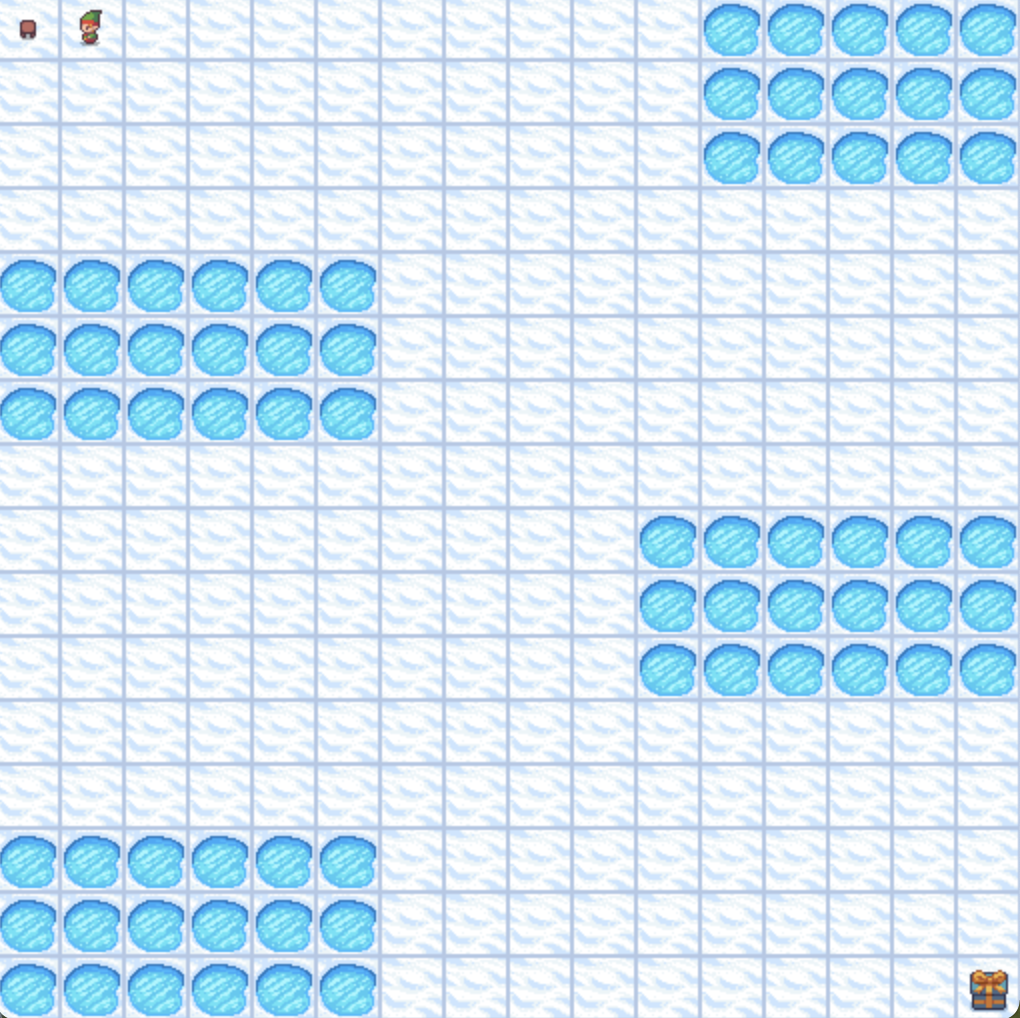}
        \caption{FL: L3}
        \label{fig:fl_l3}
    \end{subfigure}
    \vspace{3mm} 
    \caption{Frozen Lake Environments}
    \label{fig:frozen_lake_envs}
    \vspace{3mm} 
\end{figure}
\begin{figure}[!ht]
    
    \begin{subfigure}{.24\columnwidth}
        \includegraphics[width=0.9\linewidth]{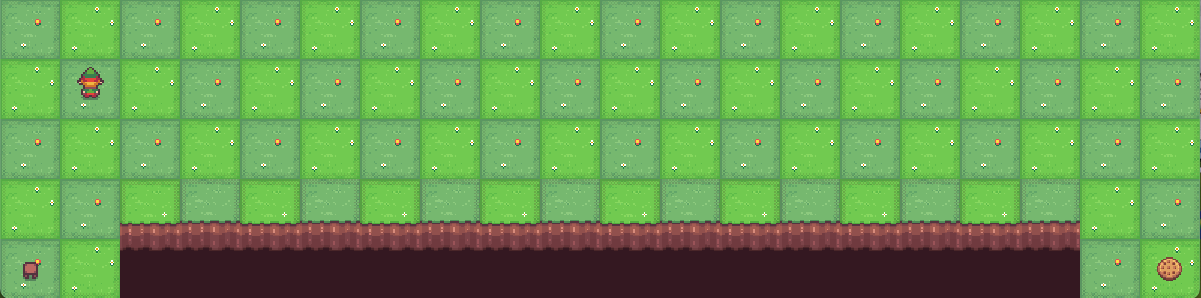}
        \caption{CW: Base}
        \label{fig:cw_s}
    \end{subfigure}%
    \begin{subfigure}{.24\columnwidth}
        \includegraphics[width=0.9\linewidth]{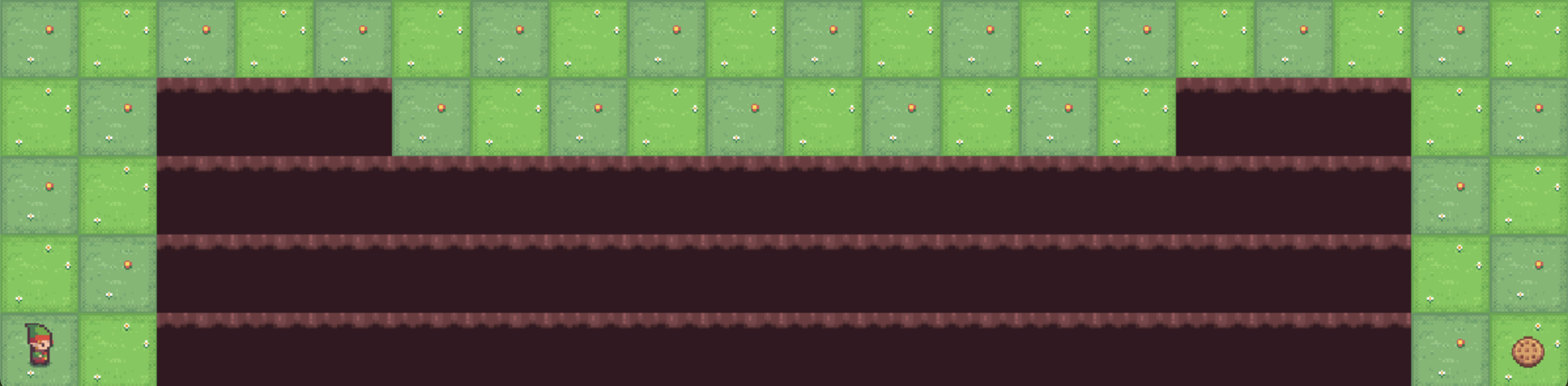}
        \caption{CW: L1}
        \label{fig:cw_l1}
    \end{subfigure}%
    \begin{subfigure}{.24\columnwidth}
        \includegraphics[width=0.9\linewidth]{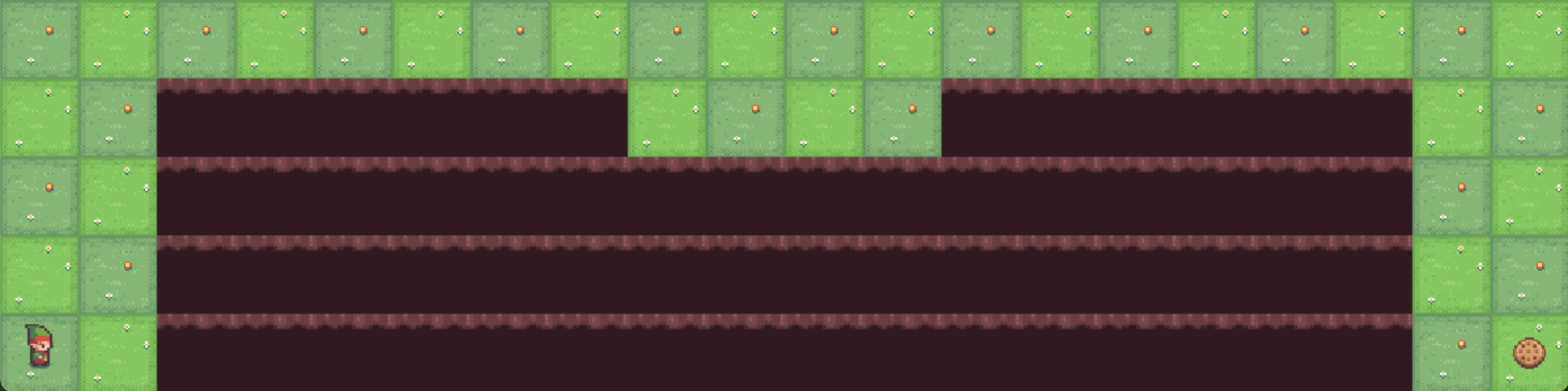}
        \caption{CW: L2}
        \label{fig:cw_l2}
    \end{subfigure}%
    \begin{subfigure}{.24\columnwidth}
        \includegraphics[width=0.9\linewidth]{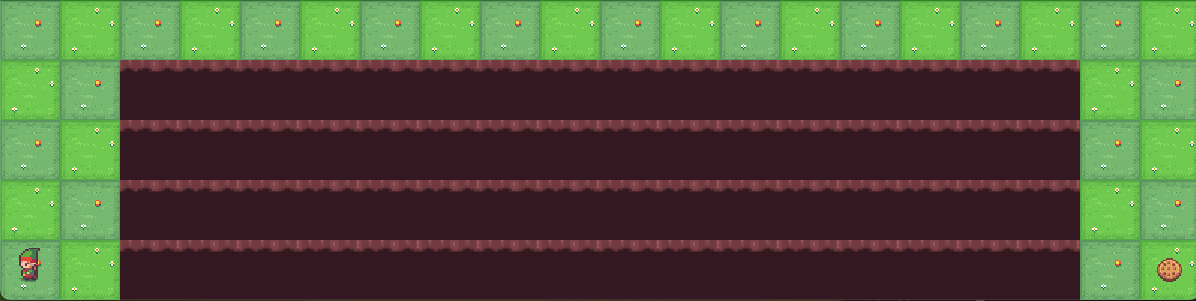}
        \caption{CW: L3}
        \label{fig:cw_l3}
    \end{subfigure}
    \vspace{3mm} 
    \caption{Cliff-Walking Environments}
    \label{fig:cw_envs}
    \vspace{3mm} 
\end{figure}
\begin{figure}[!ht]
    \begin{subfigure}{.24\columnwidth}
        \includegraphics[width=0.9\linewidth]{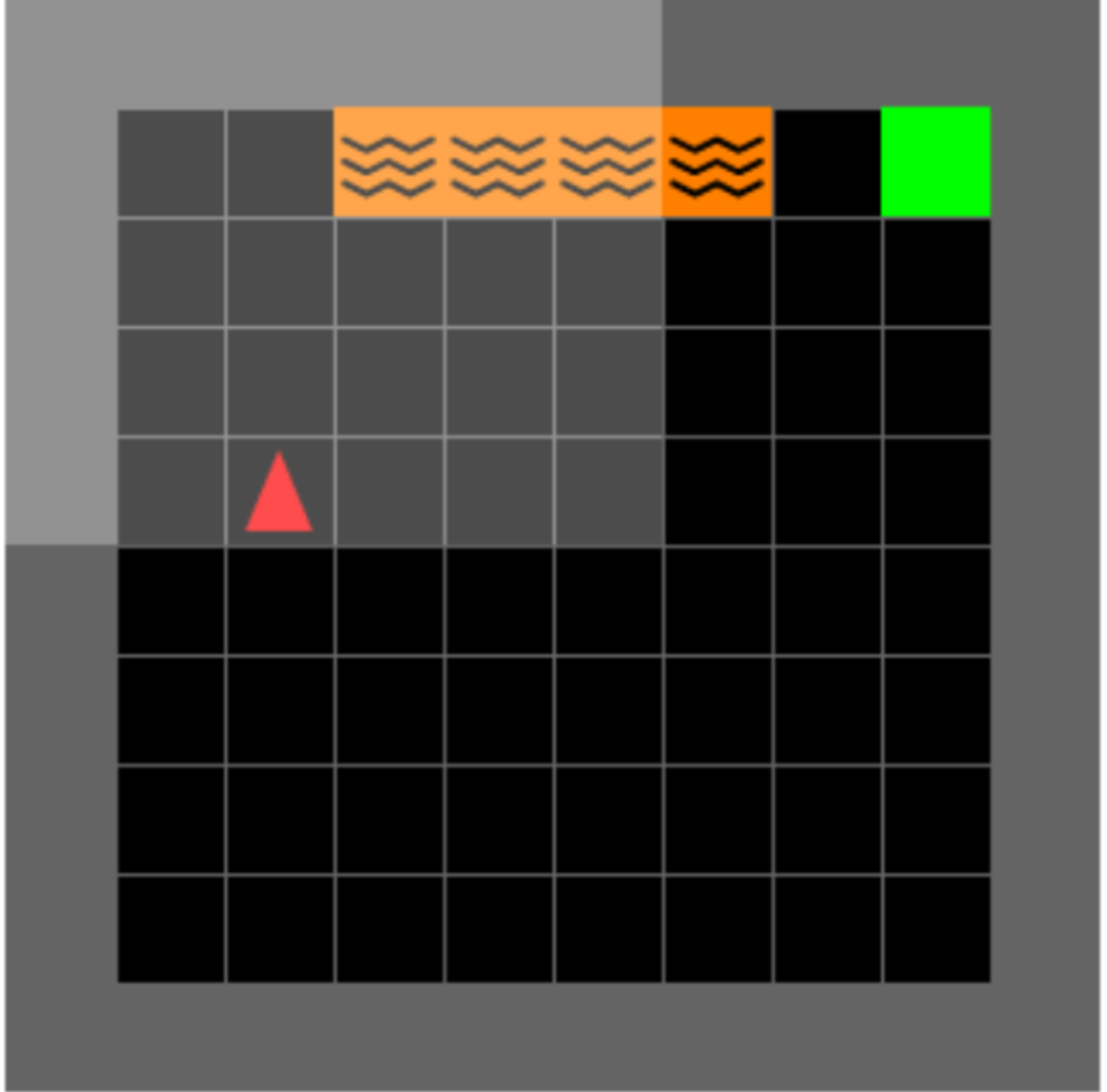}
        \caption{MGDS: B}
        \label{fig:ds_s}
    \end{subfigure}%
    \begin{subfigure}{.24\columnwidth}
        \includegraphics[width=0.9\linewidth]{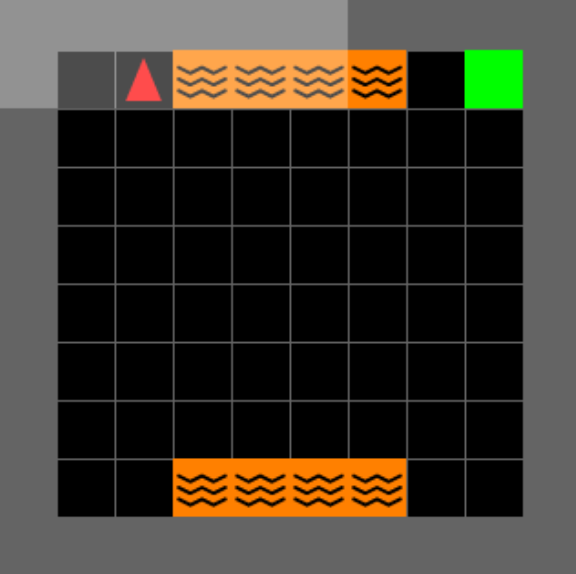}
        \caption{MGDS: L1}
        \label{fig:ds_l1}
    \end{subfigure}%
    \begin{subfigure}{.24\columnwidth}
        \includegraphics[width=0.9\linewidth]{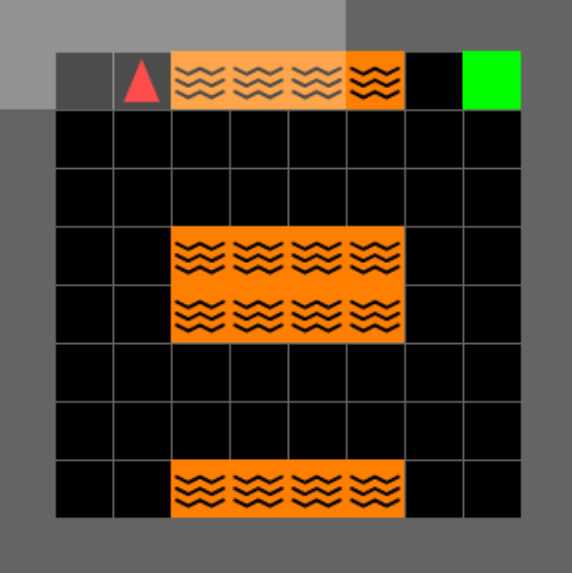}
        \caption{MGDS: L2}
        \label{fig:ds_l2}
    \end{subfigure}%
    \begin{subfigure}{.24\columnwidth}
        \includegraphics[width=0.9\linewidth]{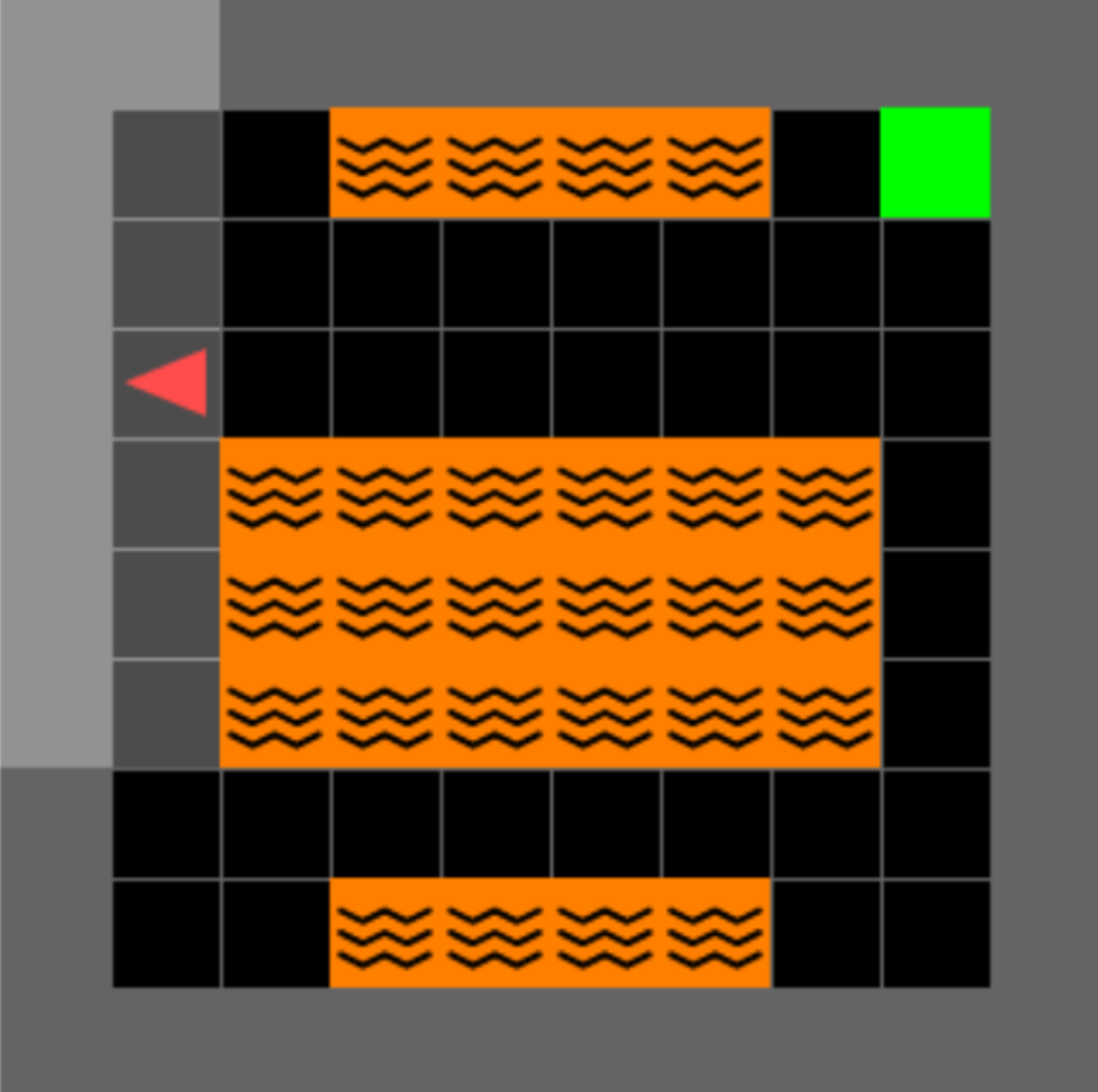}
        \caption{MGDS: L3}
        \label{fig:ds_l3}
    \end{subfigure}
    \vspace{3mm} 
    \caption{Minigrid:DistributionShift Environments}
    \label{fig:ds_envs}
    \vspace{3mm} 
\end{figure}
\begin{figure}[!ht]
    \begin{subfigure}{.24\columnwidth}
        \includegraphics[width=0.9\linewidth]{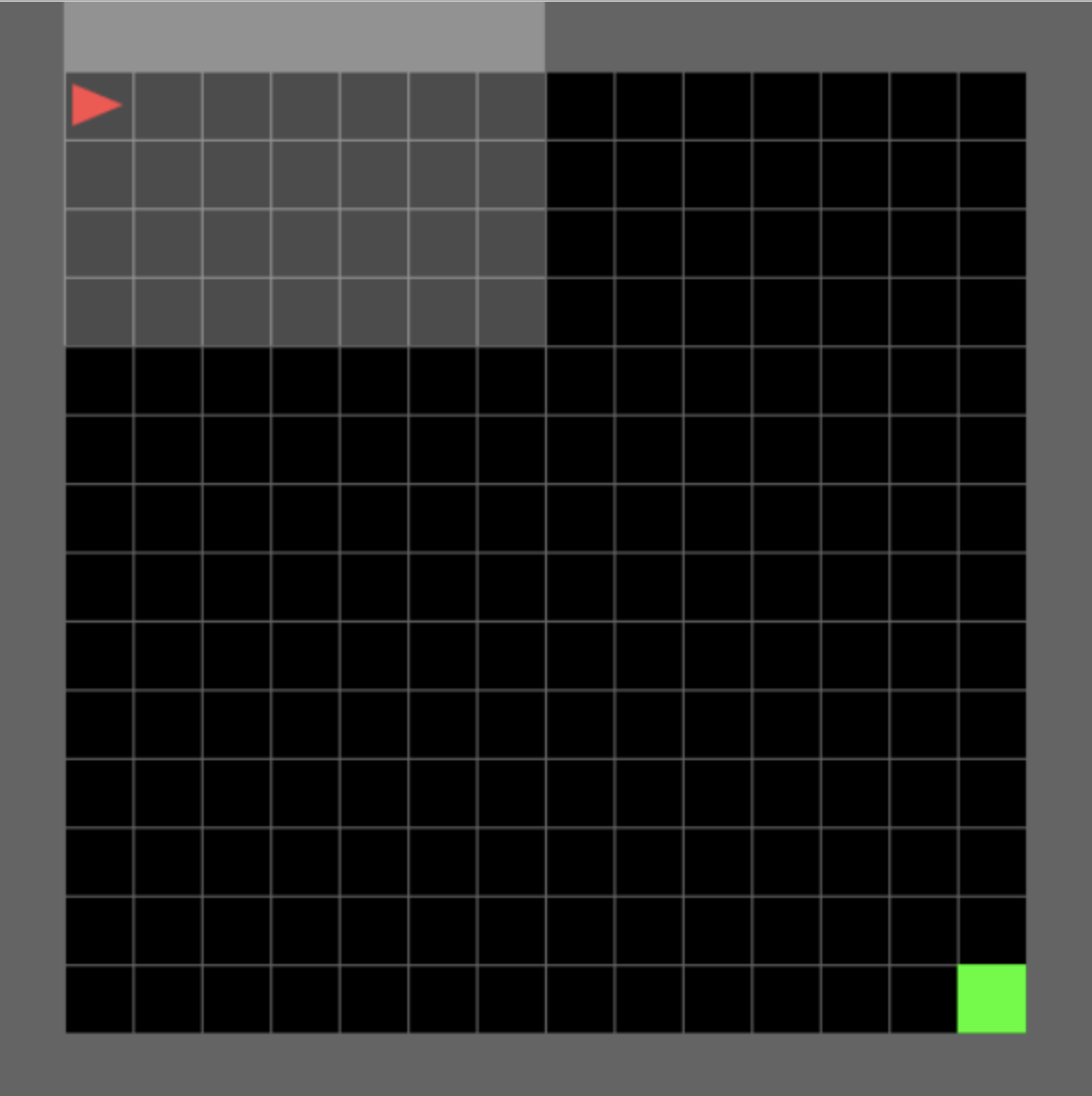}
        \caption{MGWL: B}
        \label{fig:mgwl_s}
    \end{subfigure}%
    \begin{subfigure}{.24\columnwidth}
        \includegraphics[width=0.9\linewidth]{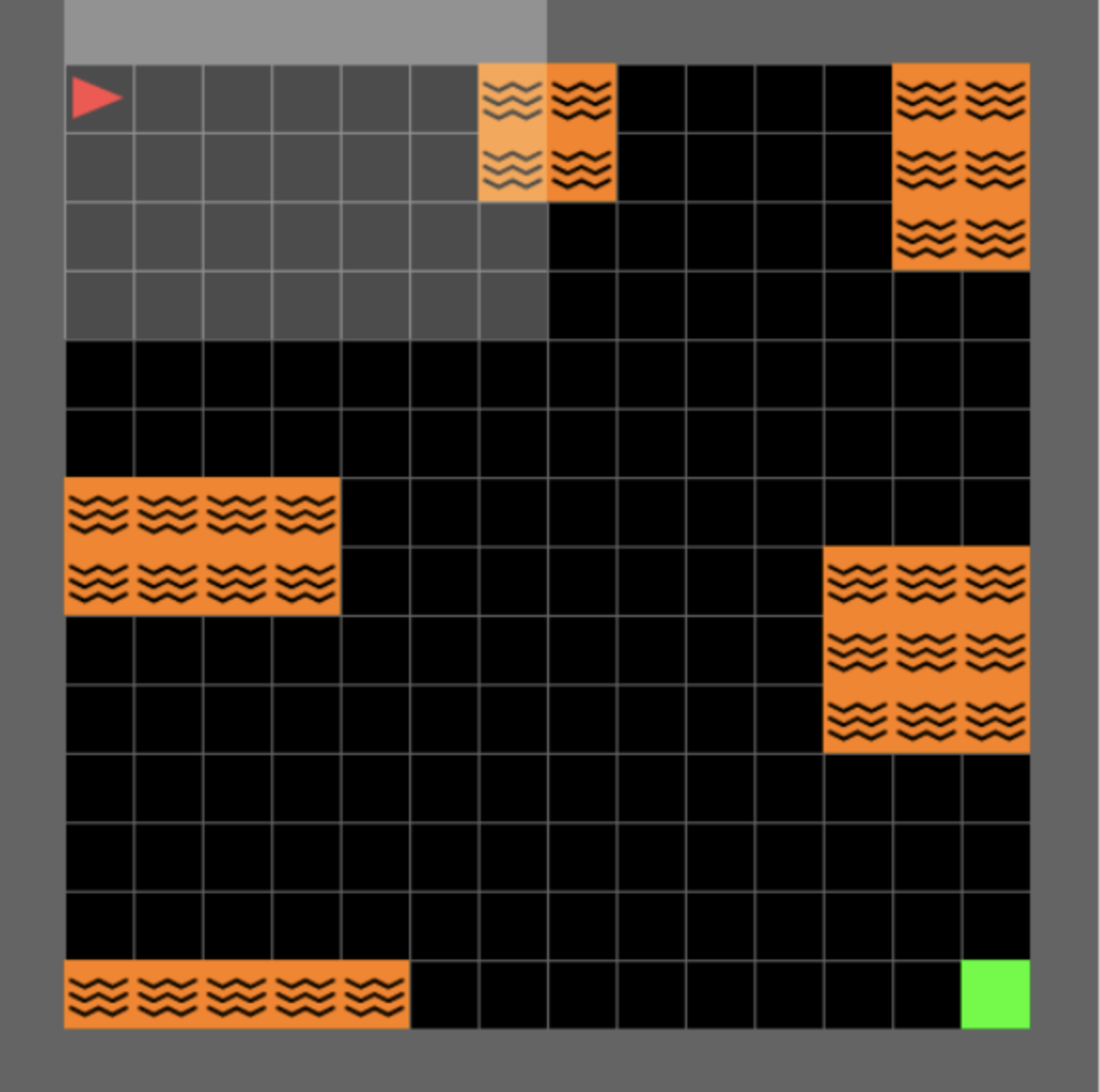}
        \caption{MGWL: L1}
        \label{fig:mgwl_l1}
    \end{subfigure}%
    \begin{subfigure}{.24\columnwidth}
        \includegraphics[width=0.9\linewidth]{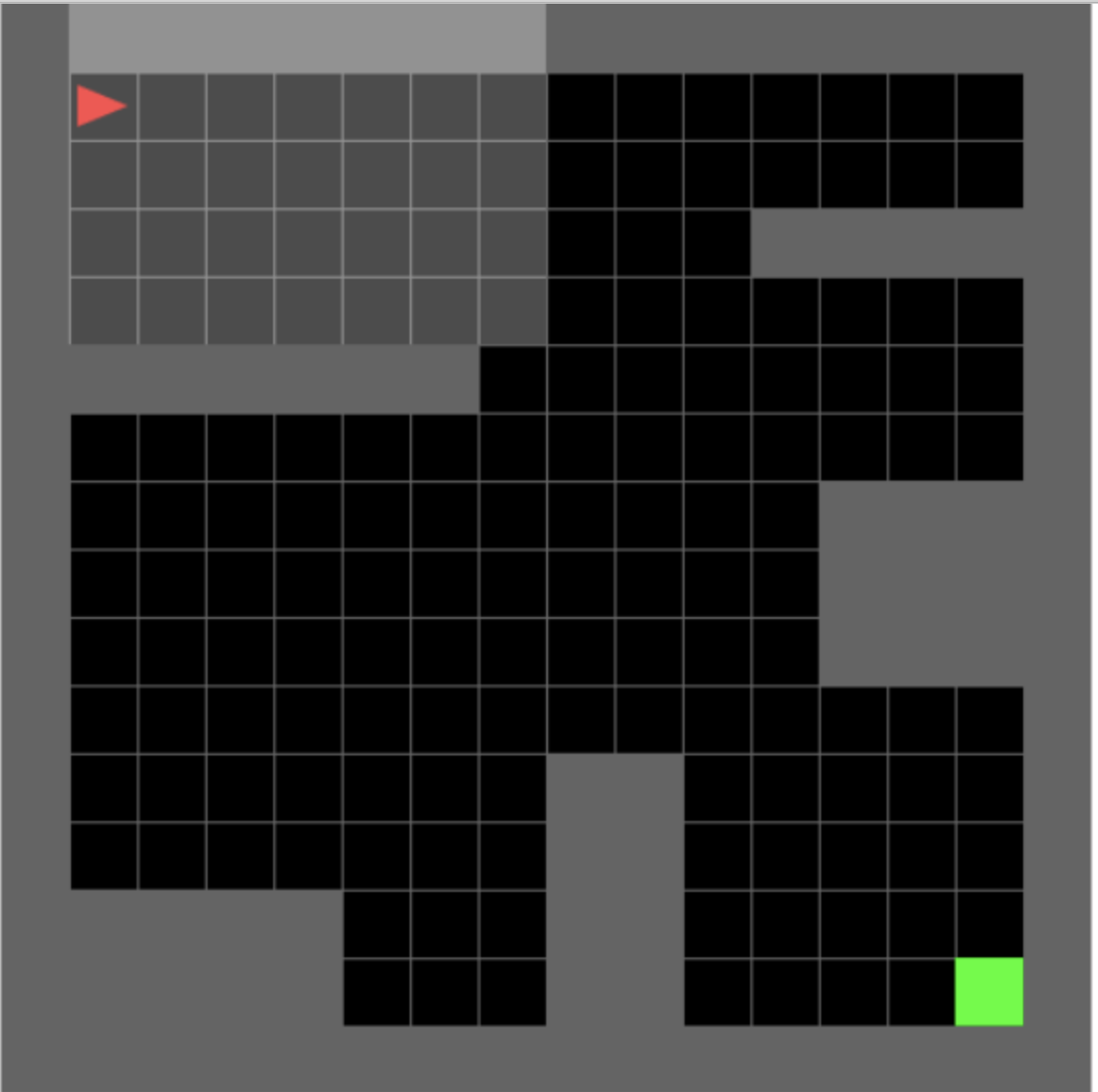}
        \caption{MGWL: L2}
        \label{fig:mgwl_l2}
    \end{subfigure}%
    \begin{subfigure}{.24\columnwidth}
        \includegraphics[width=0.9\linewidth]{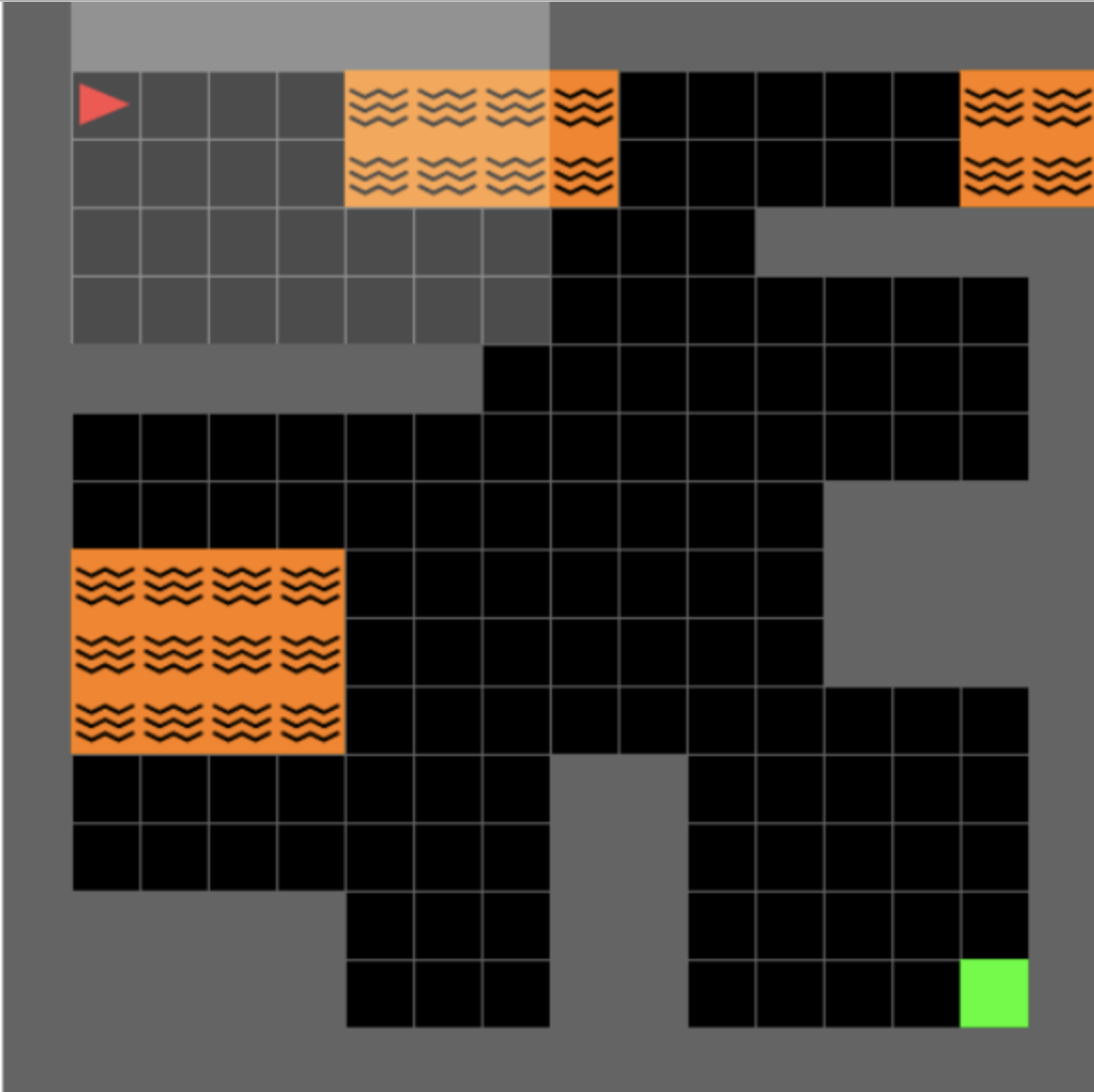}
        \caption{MGWL: L3}
        \label{fig:mgwl_l3}
    \end{subfigure}
    \vspace{3mm} 
    \caption{Minigrid:Walls\&Lava Environments}
    \label{fig:mgwl_envs}
    \vspace{3mm} 
\end{figure}

\begin{figure}[!ht]
    \begin{subfigure}{.24\columnwidth}
        \includegraphics[width=0.9\linewidth]{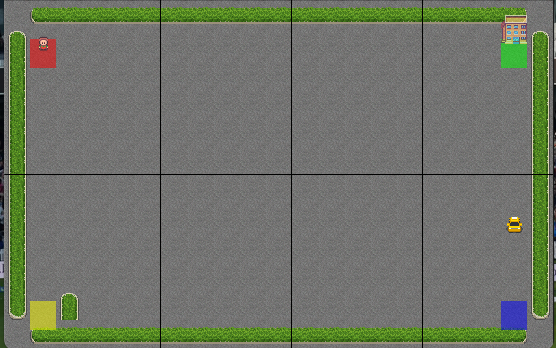}
        \caption{TX: B}
        \label{fig:taxi_s}
    \end{subfigure}%
    \begin{subfigure}{.24\columnwidth}
        \includegraphics[width=0.9\linewidth]{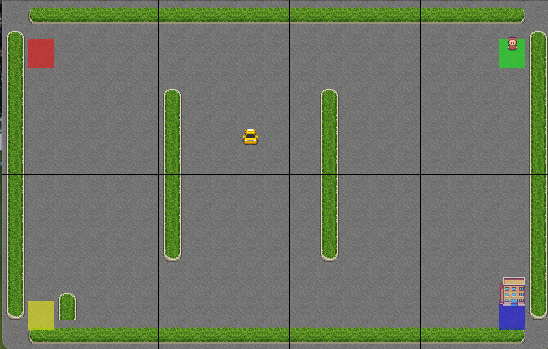}
        \caption{TX: L1}
        \label{fig:taxi_l1}
    \end{subfigure}%
    \begin{subfigure}{.24\columnwidth}
        \includegraphics[width=0.9\linewidth]{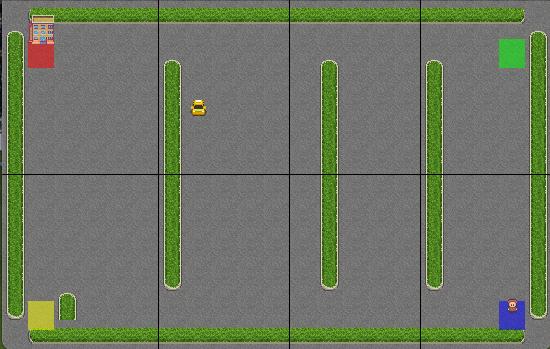}
        \caption{TX: L2}
        \label{fig:taxi_l2}
    \end{subfigure}%
    \begin{subfigure}{.24\columnwidth}
        \includegraphics[width=0.9\linewidth]{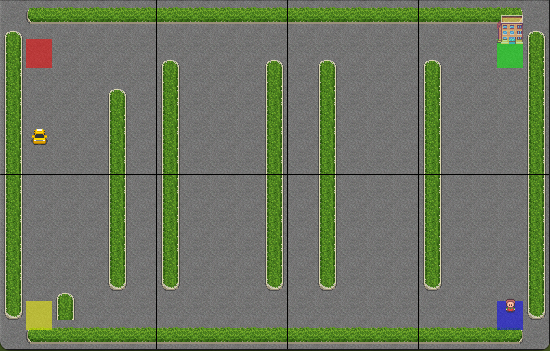}
        \caption{TX: L3}
        \label{fig:taxi_l3}
    \end{subfigure}
    \vspace{3mm} 
    \caption{Taxi Environments}
    \label{fig:taxi_envs}
    \vspace{3mm} 
\end{figure}

%% file: results.tex
\begin{figure*}[ht!]
\centering
    \begin{subfigure}{.32\textwidth}
        \centering
        \includegraphics[height=3.5cm]{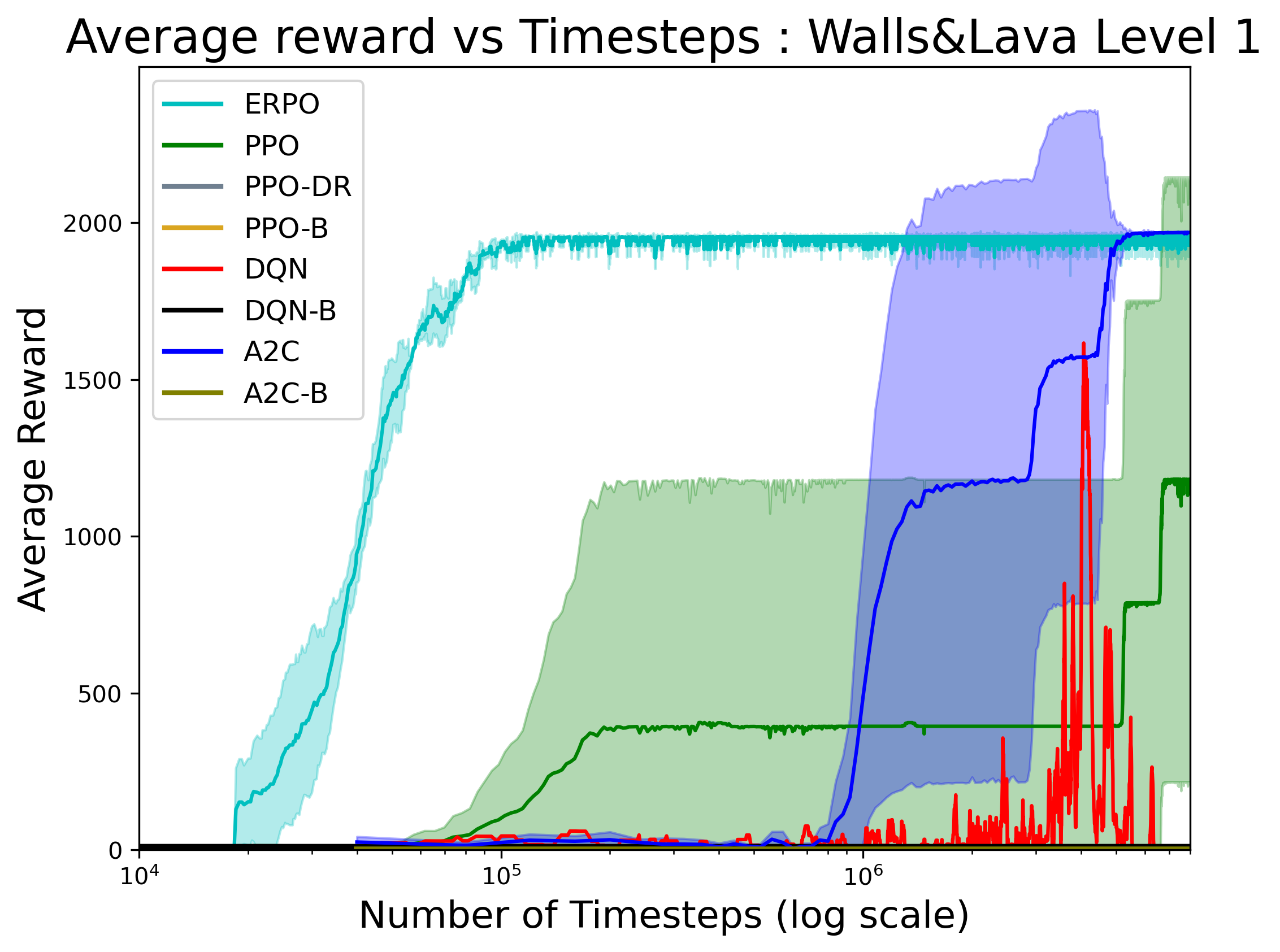}
        \caption{Level 1}
        \label{fig:mgwl_l1}
    \end{subfigure}
    \begin{subfigure}{.32\textwidth}
        \centering
        \includegraphics[height=3.5cm]{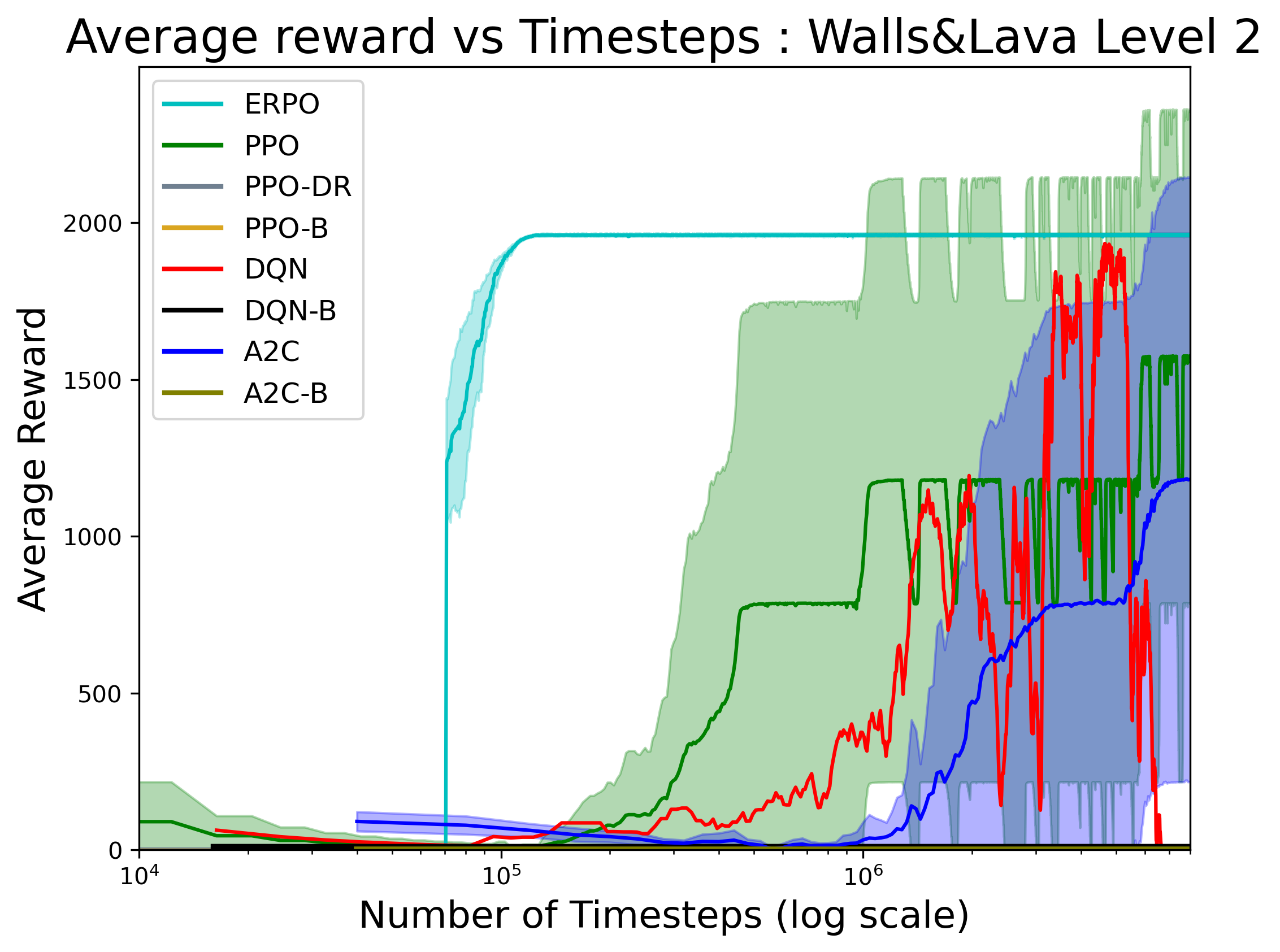}
        \caption{Level 2}
        \label{fig:mgwl_l2}
    \end{subfigure}
    \begin{subfigure}{.32\textwidth}
        \centering
        \includegraphics[height=3.5cm]{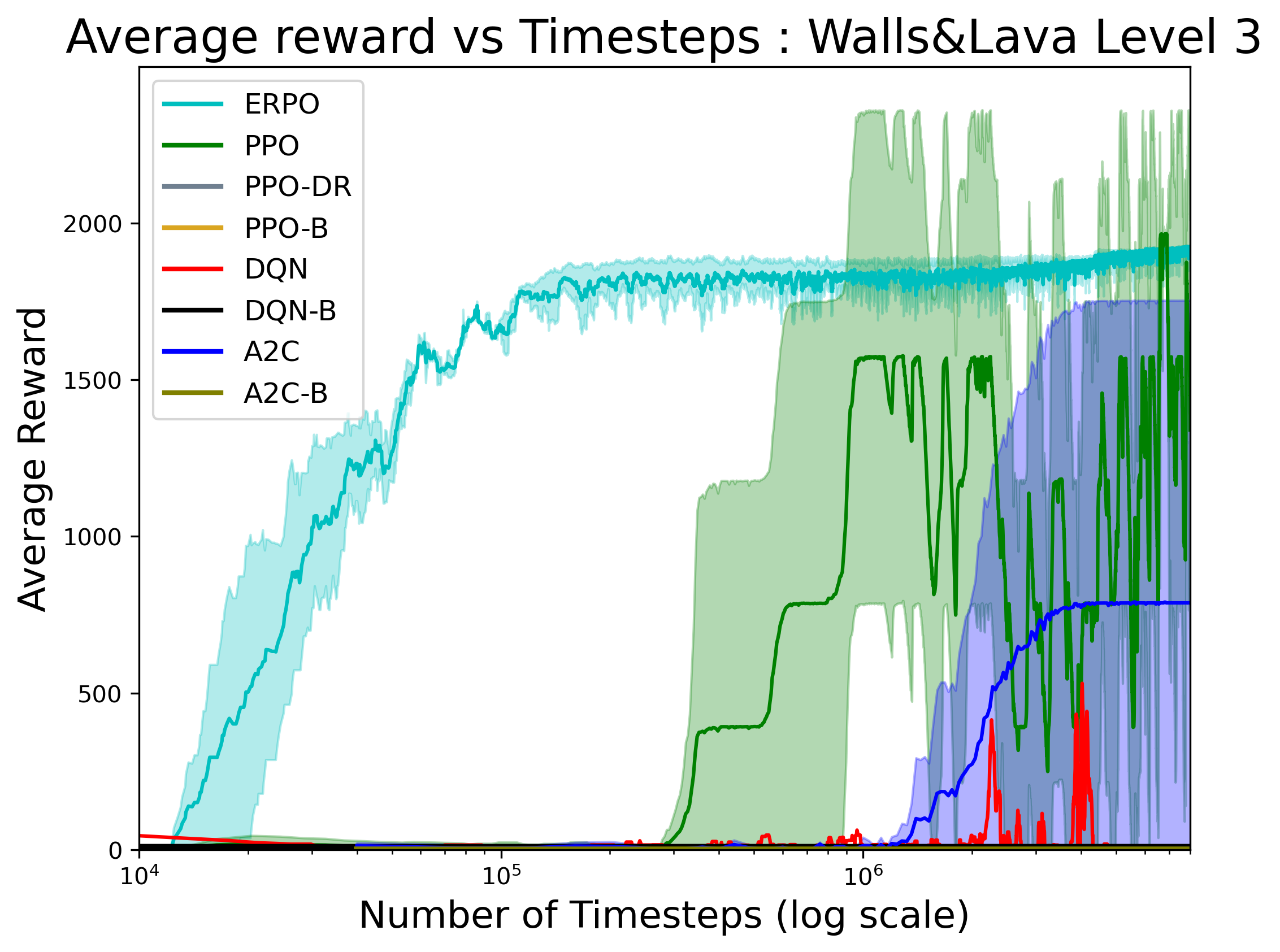}
        \caption{Level 3}
        \label{fig:mgwl_l3}
    \end{subfigure}
        \vspace{3mm}
\caption{Walls\&Lava Environments: Subfigures (a), (b), and (c) indicate the results for levels 1, 2 and 3 (see Fig. 6). The base figure is the original environment, and levels 1 to 3 indicate versions of the environment with certain features altered that induce progressively increasing distribution shift. The agent has to navigate a gridworld of increasing complexity with lava, walls or both. It has 2000 timesteps to reach the goal ($r = 0$ for each step) and gets reward $r = 2000 - t$ upon reaching the goal (where $t$ is the current timestep). ERPO surpasses other models, with steadied performance even as levels increase. } 
\vspace{2mm}
\label{fig:mgwl_results}
\end{figure*}

% Figure 2 for FL Environment
\begin{figure*}[ht!]
\centering
    \begin{subfigure}{.32\textwidth}
        \centering
        \includegraphics[height=3.5cm]{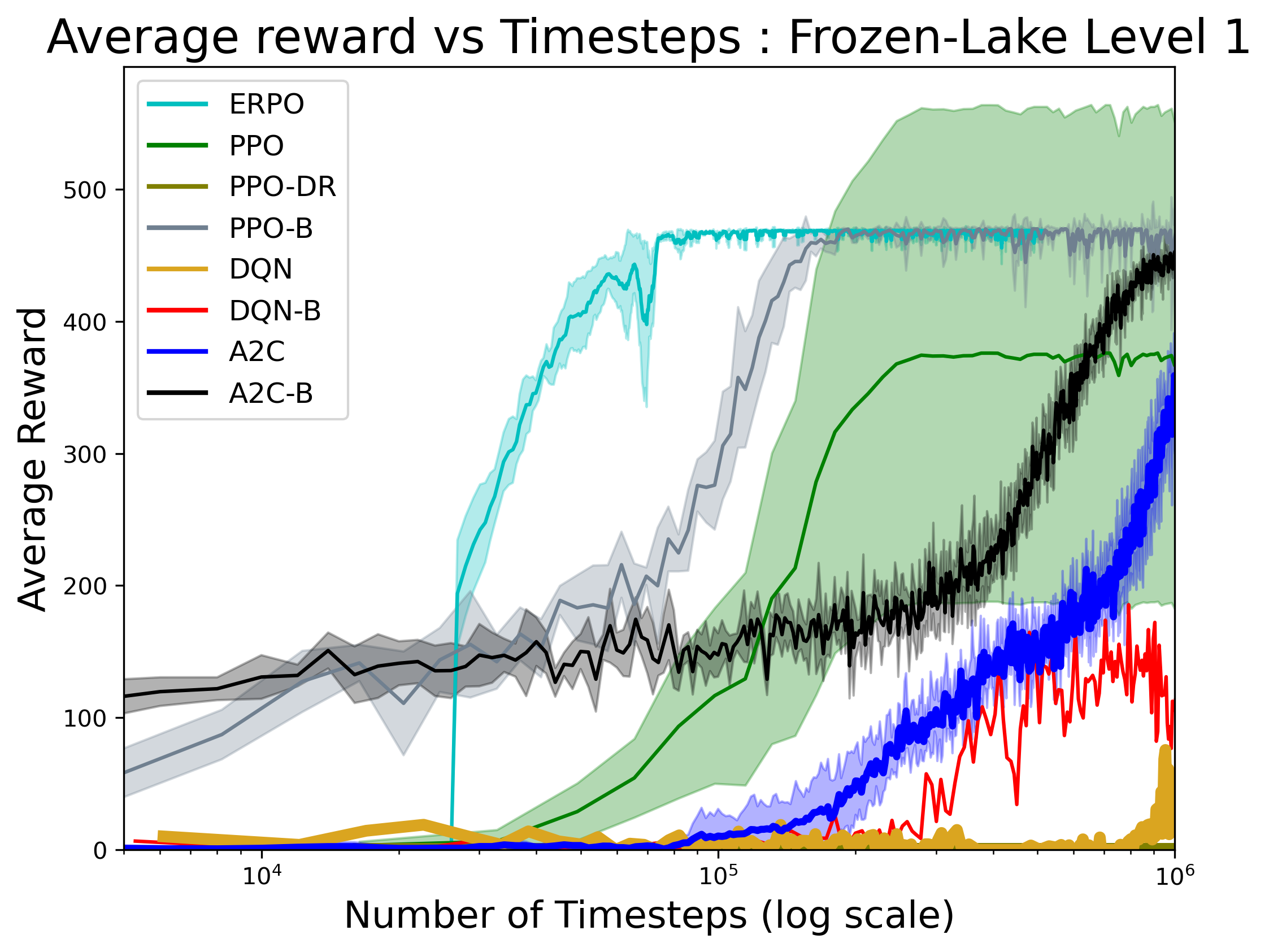}
        \caption{Level 1}
        \label{fig:fl_l1}
    \end{subfigure}
    \begin{subfigure}{.32\textwidth}
        \centering
        \includegraphics[height=3.5cm]{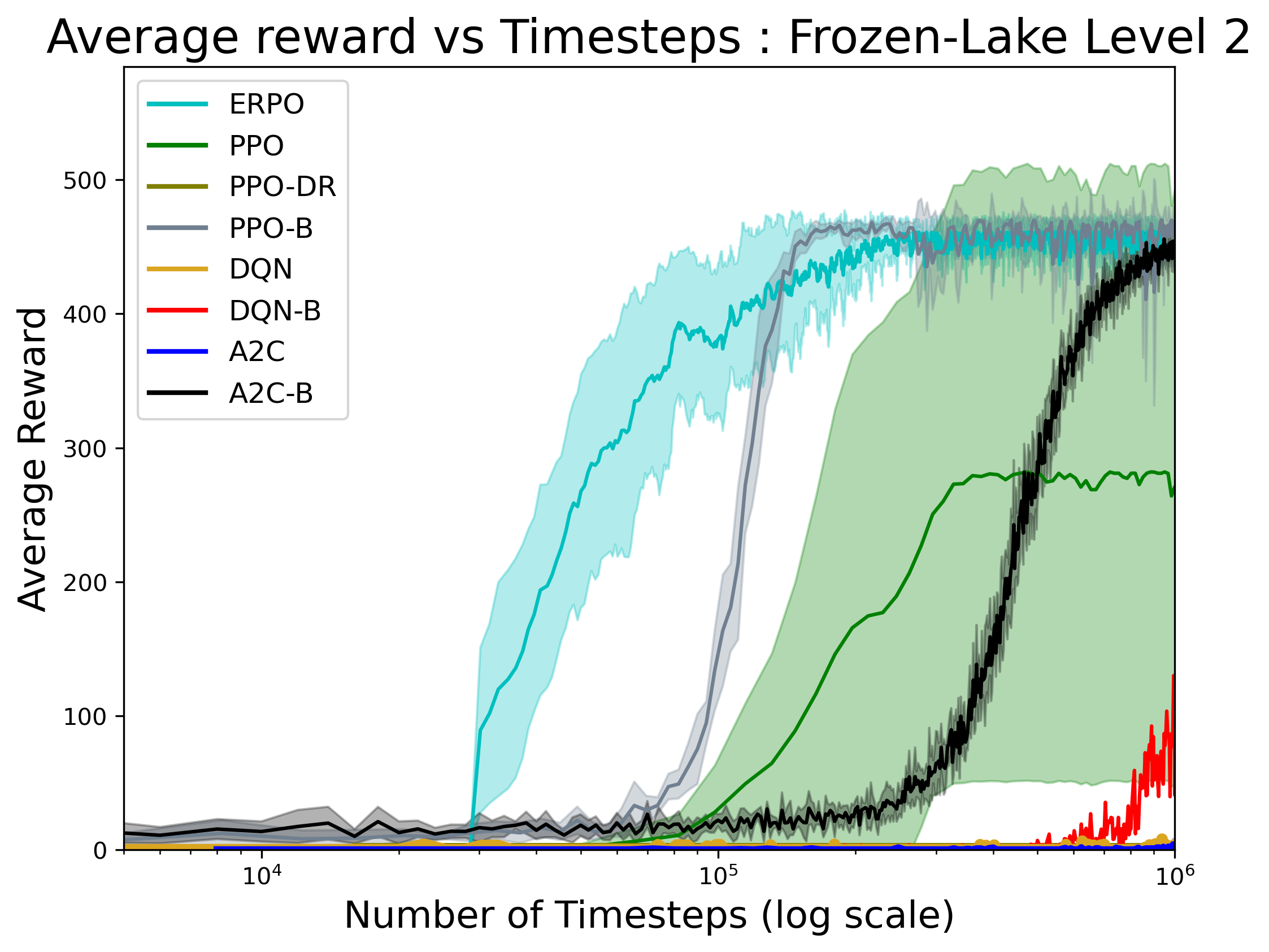}
        \caption{Level 2}
        \label{fig:fl_l2}
    \end{subfigure}
    \begin{subfigure}{.32\textwidth}
        \centering
        \includegraphics[height=3.5cm]{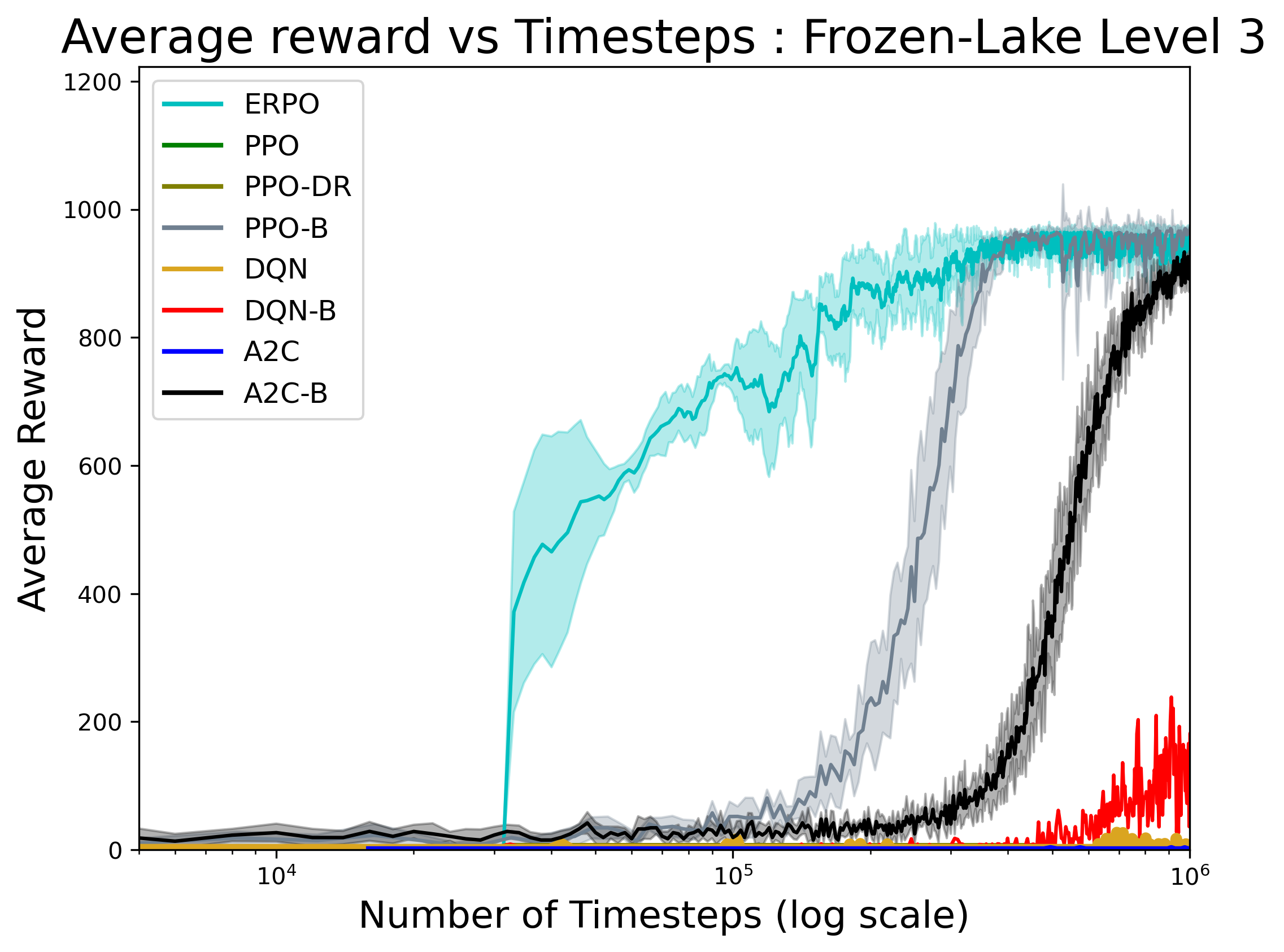}
        \caption{Level 3}
        \label{fig:fl_l3}
    \end{subfigure}
    \vspace{3mm}
\caption{Frozen-Lake Environments: (See Fig. 3) The agent (elf) is given 500 steps to reach the goal ($r = 0$ for each step) and given a reward $r = 500 - t$ upon reaching the goal (where $t$ is the current timestep). ERPO 
consistently maintains a strong performance, particularly at higher levels. PPO-DR and A2-C B demonstrate robustness. }
\vspace{2mm}
\label{fig:fl_results}
\end{figure*}

% Figure 3 for CW Environment
\begin{figure*}[ht!]
\centering
    \begin{subfigure}{.32\textwidth}
        \centering
        \includegraphics[height=3.5cm]{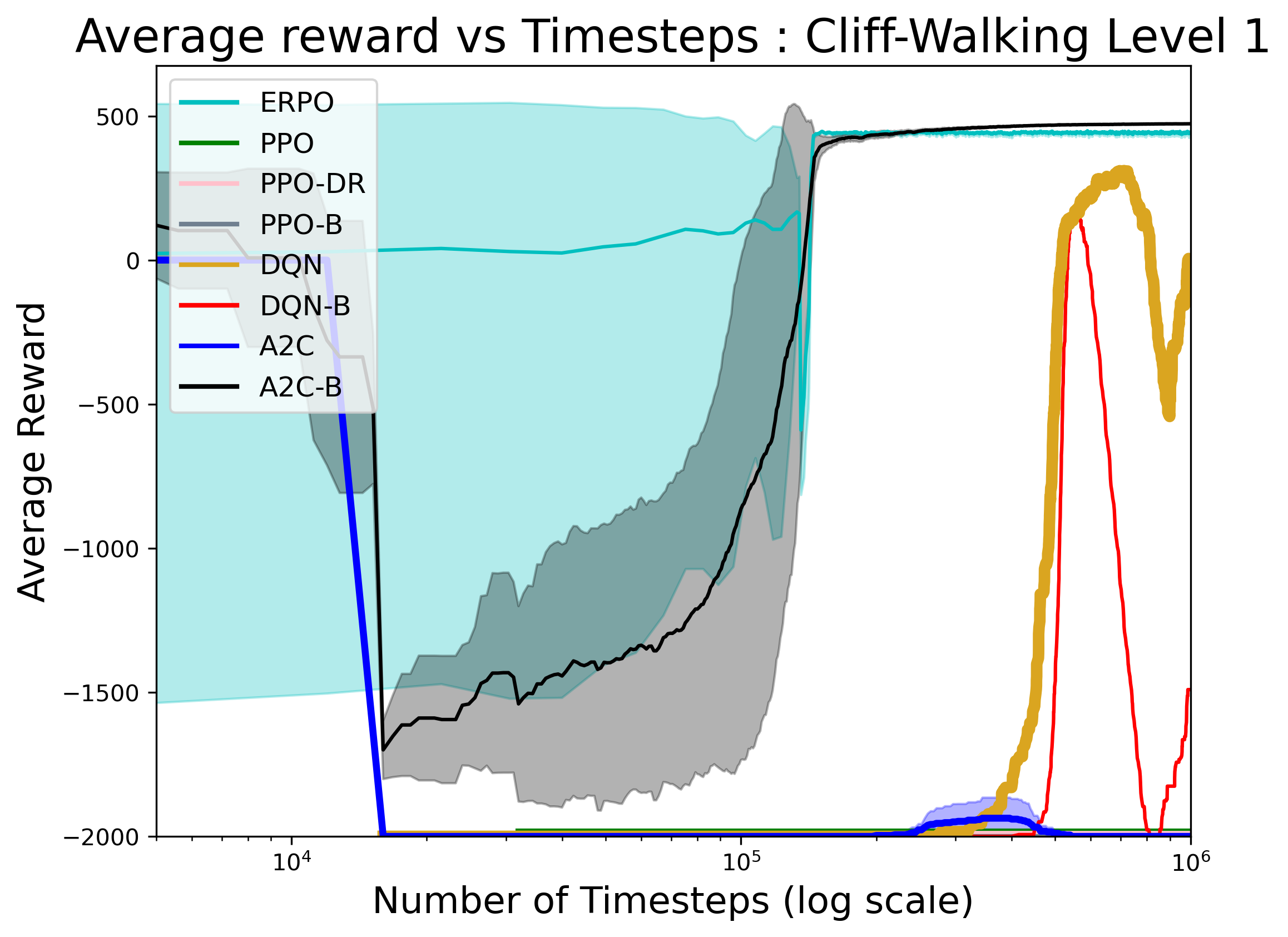}
        \caption{Level 1}
        \label{fig:cw_l1}
    \end{subfigure}
    \begin{subfigure}{.32\textwidth}
        \centering
        \includegraphics[height=3.5cm]{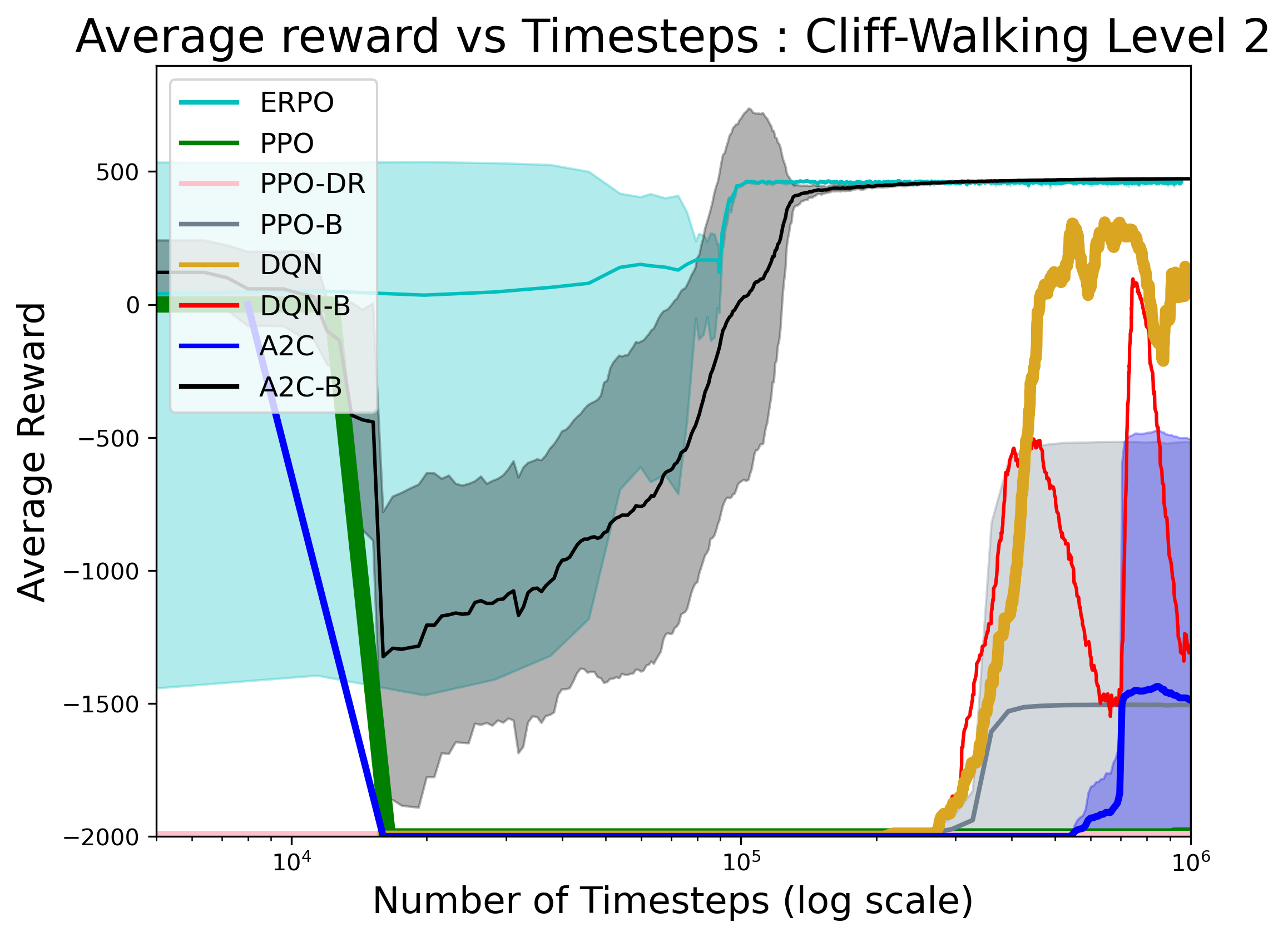}
        \caption{Level 2}
        \label{fig:cw_l2}
    \end{subfigure}
    \begin{subfigure}{.32\textwidth}
        \centering
        \includegraphics[height=3.5cm]{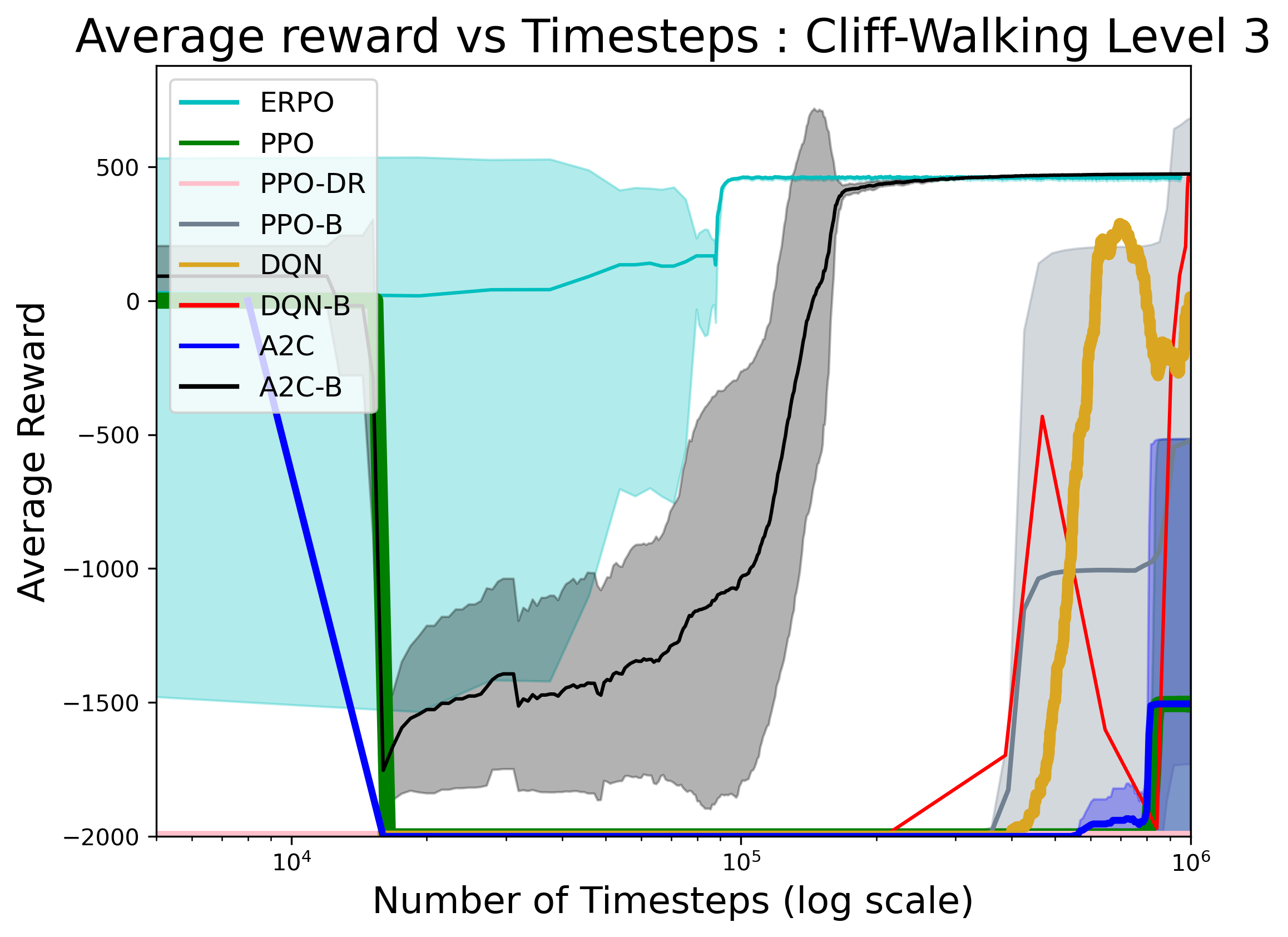}
        \caption{Level 3}
        \label{fig:cw_l3}
    \end{subfigure}
        \vspace{3mm}
\caption{Cliff-Walking Environments: (See Fig. 4) The agent is allowed 2000 steps to get from the start to the goal ($r = -1$ for each step) and given a reward $r = + 500$ for reaching the goal. ERPO  outperforms other methods, at higher levels. Notably, baselines trained over pre-learned models tend to fluctuate, indicating challenges in adapting to new environments. A2C-B comes close and outperforms ERPO in Level 1.}
\vspace{2mm}
\label{fig:cw_results}
\end{figure*}

We present the results of each environment for ERPO and the other baseline algorithms. We note that the performance of ERPO does not vary much with increasing levels of 
difficulty, even when the new environment is drastically different and much more difficult to navigate than the base environment, while the other algorithms suffer.

\mypara{Comparison with models trained from scratch and Domain Randomization:}
 ERPO significantly outperforms the other algorithms in terms of timesteps required for convergence. PPO-DR and A2C are the closest competitors, yet they still require up to an order of magnitude more timesteps than ERPO in the Walls\&Lava environment. \footnote{A2C shows results later than the other algorithms in the Walls\&Lava environment because of large batch size over 8 environments, so results are indicated only after batch\_size $\times$ 8 timesteps.} The results from the Taxi environment indicate that PPO-DR has the closest performance to ERPO but still takes longer to converge. In the modified CliffWalking environments, most algorithms struggle to converge, highlighting the increased difficulty due to the larger cliff area and being returned to the starting position for entering it. The FrozenLake environment presents a more favorable scenario for A2C and PPO, but they worsen significantly as the levels increase, empirically demonstrating that ERPO adapts better.

% Figure 4 for MGDS Environment
\begin{figure*}[ht!]
\centering
    \begin{subfigure}{.32\textwidth}
        \centering
        \includegraphics[height=3.5cm]{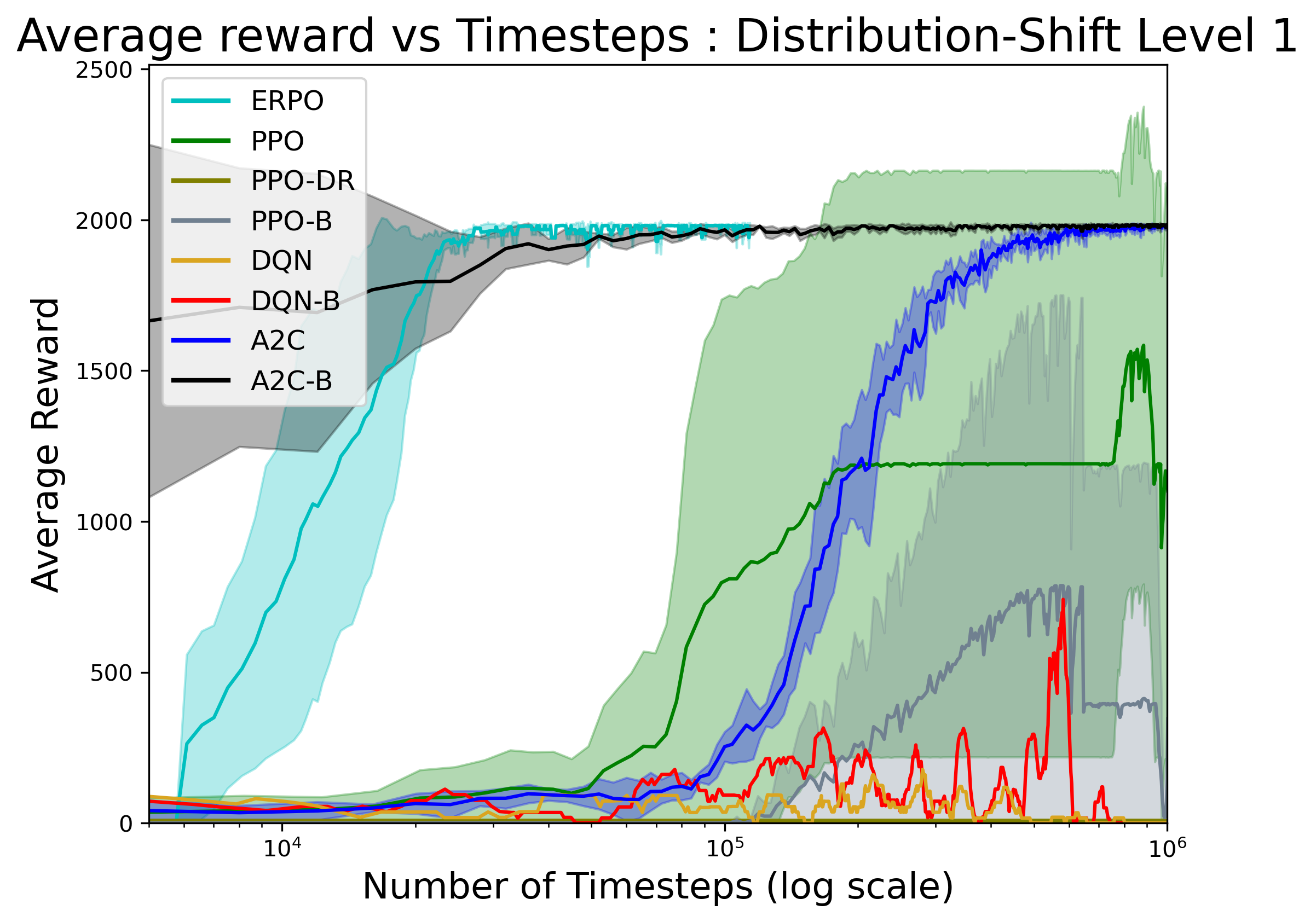}
        \caption{Level 1}
        \label{fig:ds_l1}
    \end{subfigure}
    \begin{subfigure}{.32\textwidth}
        \centering
        \includegraphics[height=3.5cm]{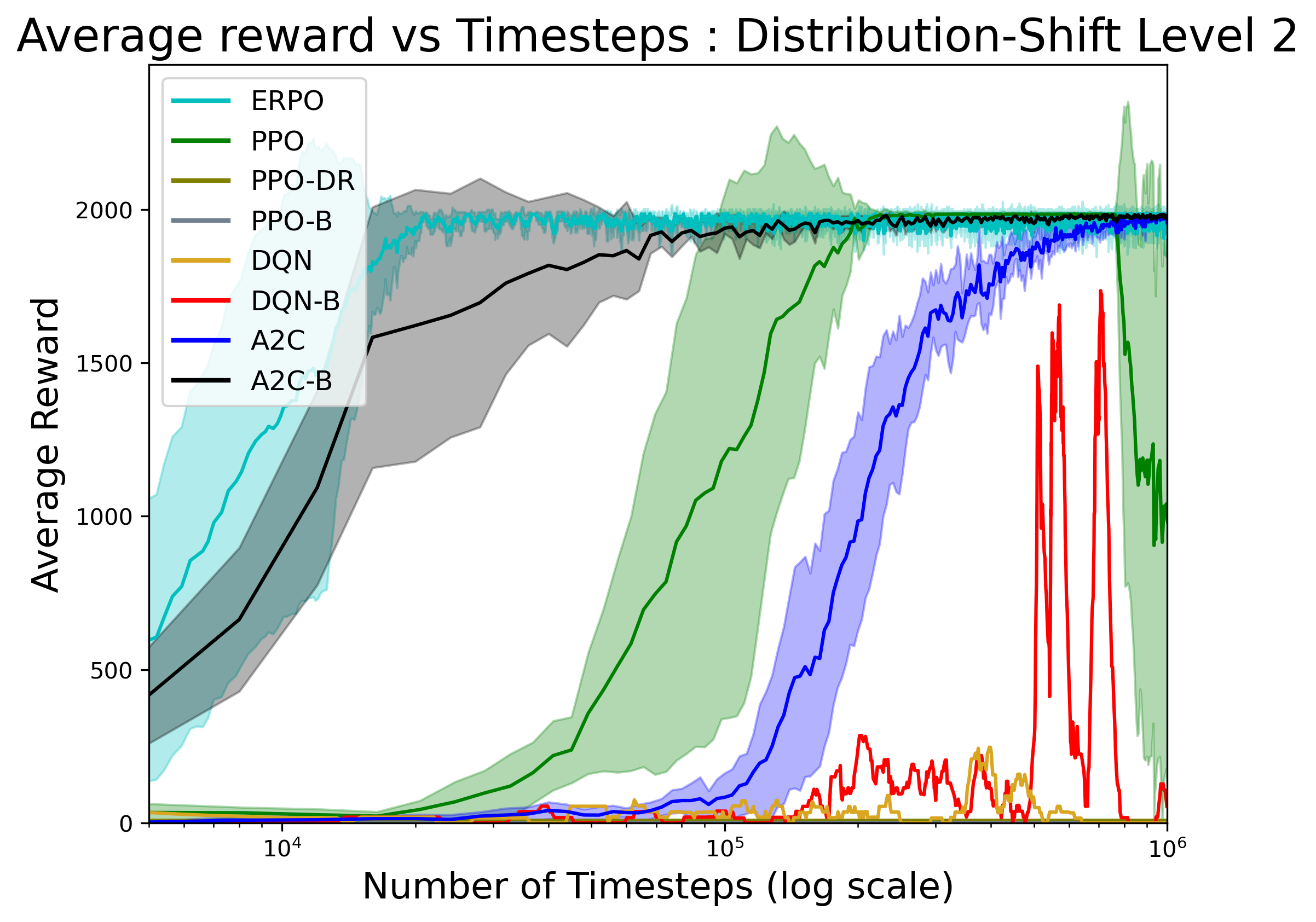}
        \caption{Level 2}
        \label{fig:ds_l2}
    \end{subfigure}
    \begin{subfigure}{.32\textwidth}
        \centering
        \includegraphics[height=3.5cm]{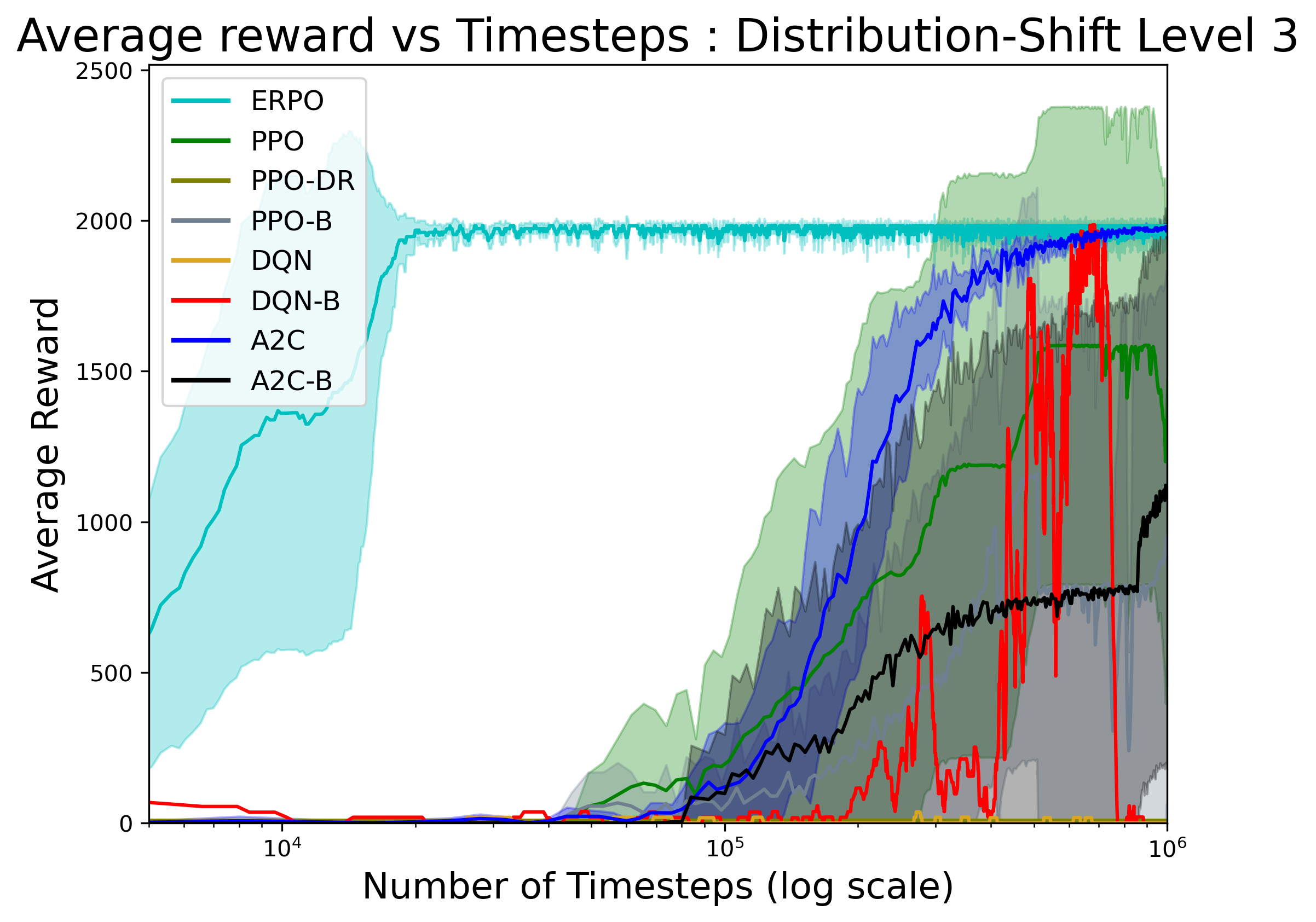}
        \caption{Level 3}
        \label{fig:ds_l3}
    \end{subfigure}
        \vspace{3mm}
\caption{Distribution-Shift Environments:  (See Fig. 5) The agent has to navigate a gridworld of increasing complexity with lava. It has 2000 timesteps to reach the goal ($r = 0$ for each step) and $r = 2000 - t$ upon reaching the goal (where $t$ is the current timestep). ERPO outperforms other methods while A2C-B comes close. Besides that PPO and A2C perform well too. The increased environmental complexity from Level 1 to Level 3 is evident, with all models facing greater challenges as the level increases.}
\label{fig:ds_results}
\end{figure*}

% Figure 5 for TX Environment
\begin{figure*}[ht!]
\centering
    \begin{subfigure}{.32\textwidth}
        \centering
        \includegraphics[height=3.5cm]{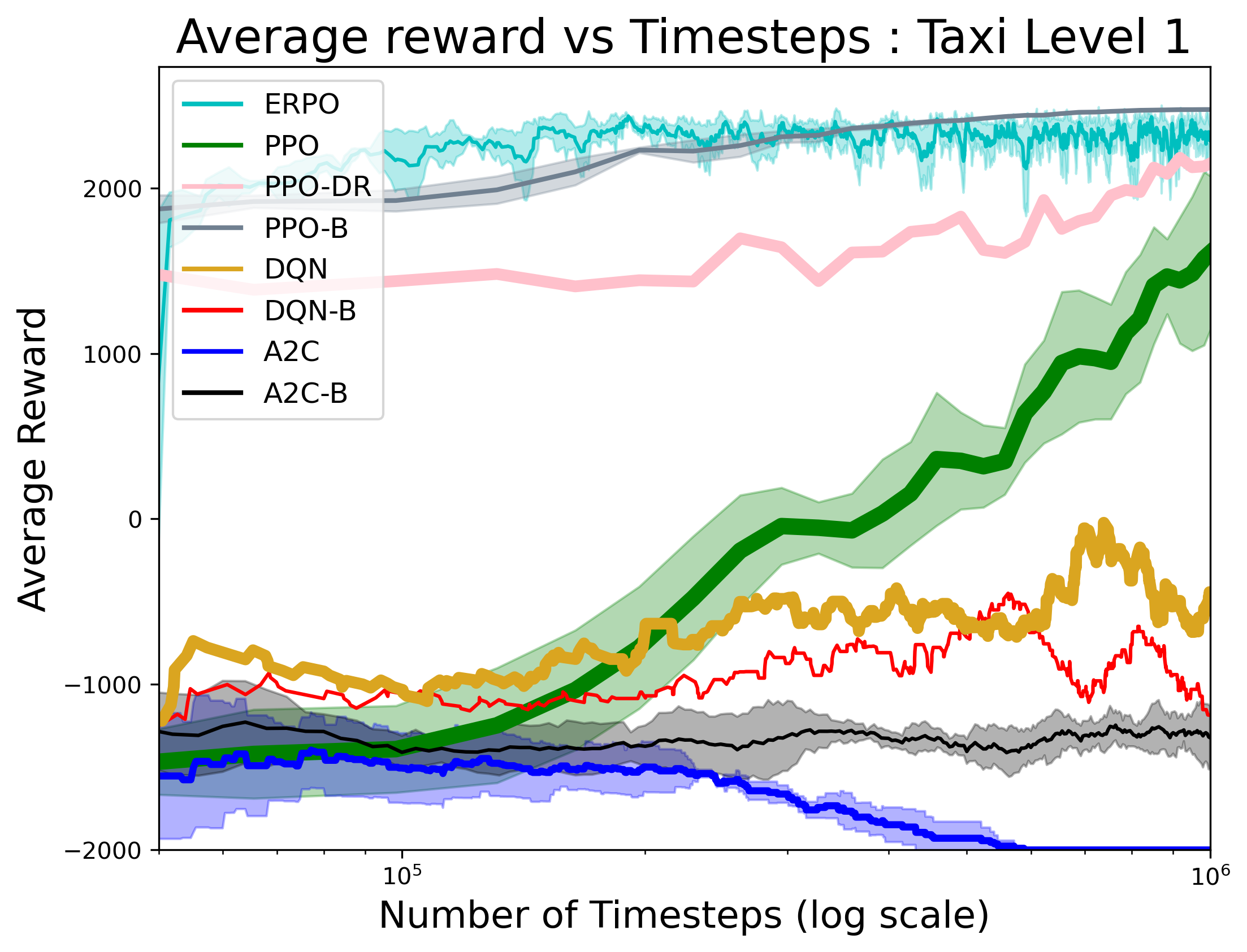}
        \caption{Level 1}
        \label{fig:taxi_l1}
    \end{subfigure}
    \begin{subfigure}{.32\textwidth}
        \centering
        \includegraphics[height=3.5cm]{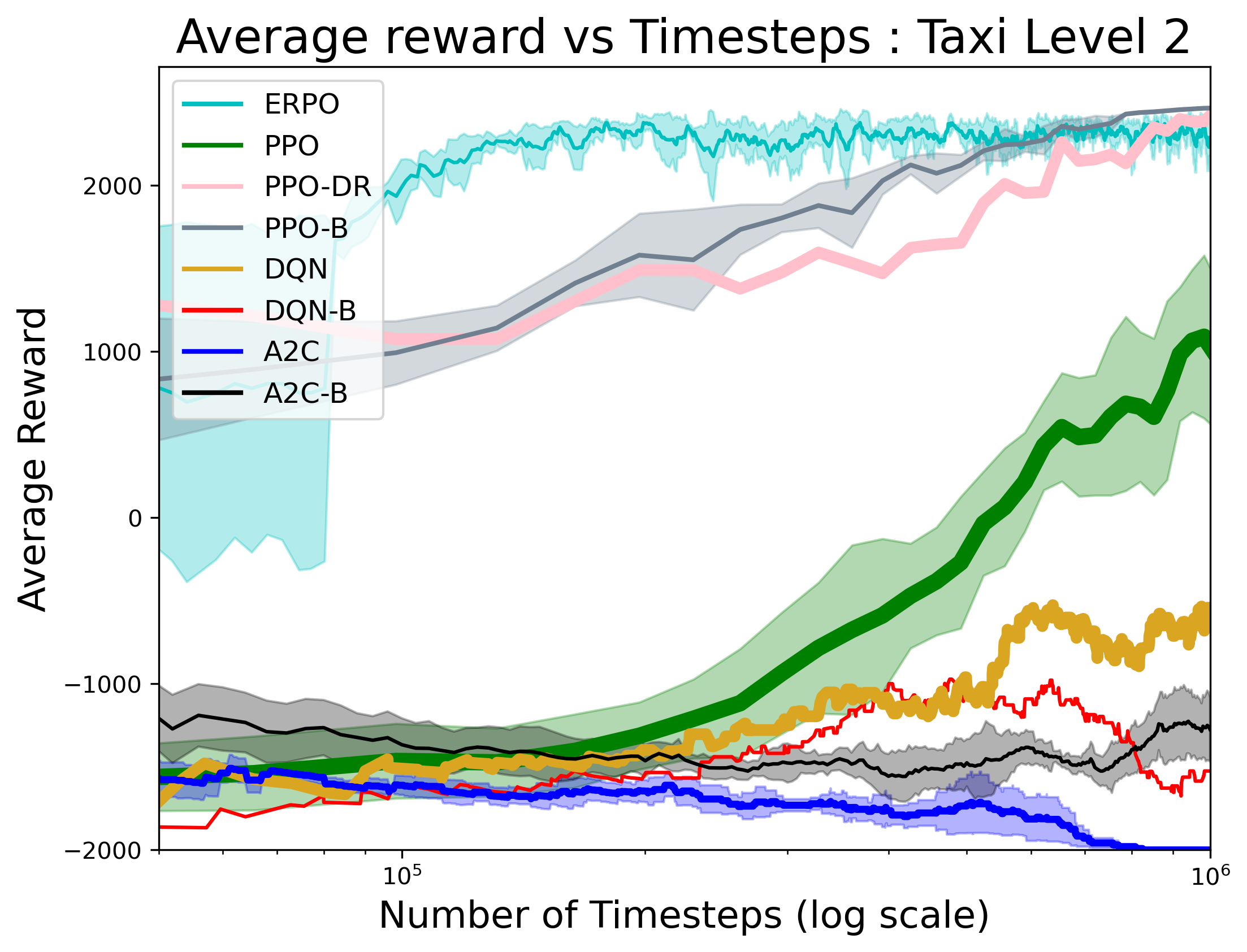}
        \caption{Level 2}
        \label{fig:taxi_l2}
    \end{subfigure}
    \begin{subfigure}{.32\textwidth}
        \centering
        \includegraphics[height=3.5cm]{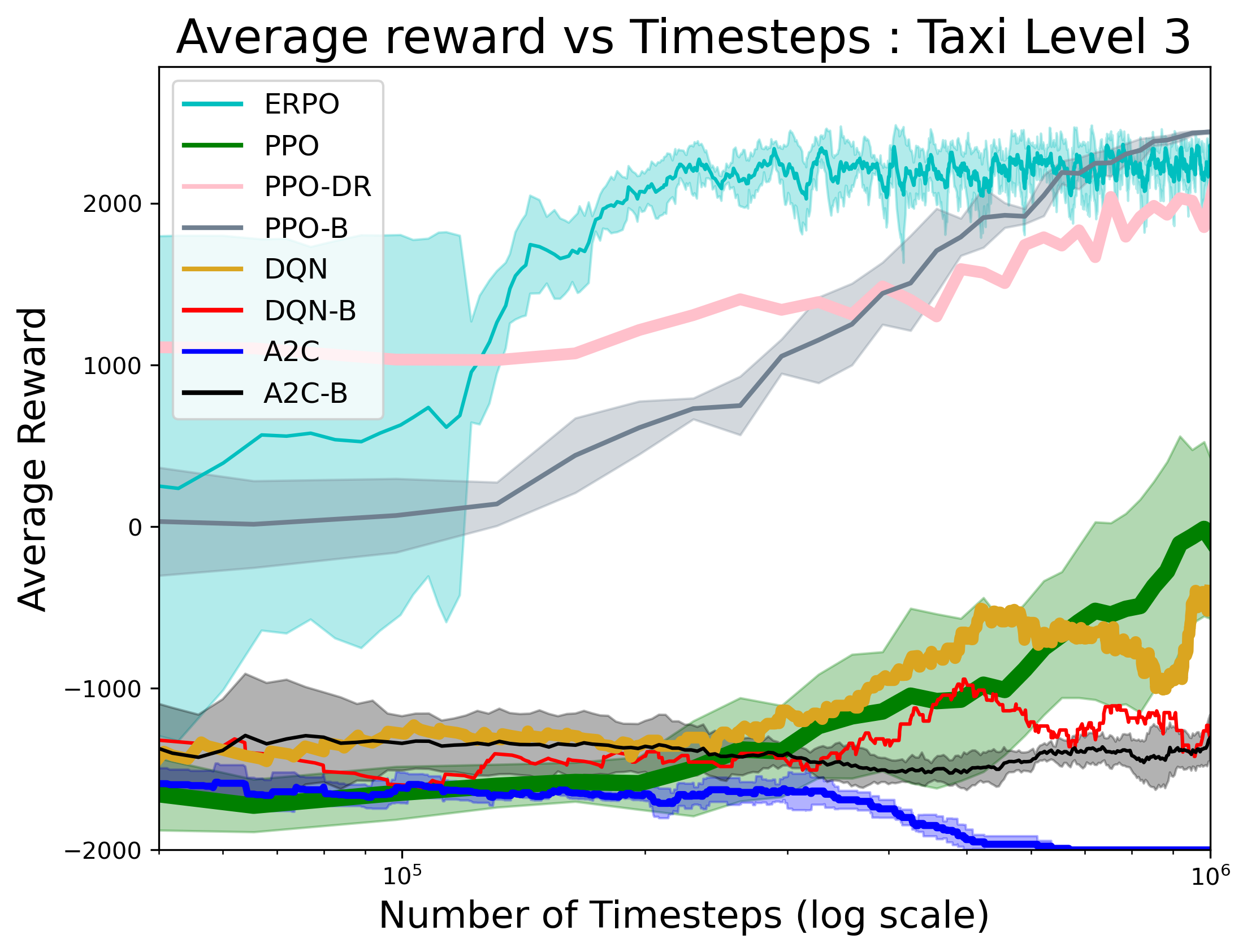}
        \caption{Level 3}
        \label{fig:taxi_l3}
    \end{subfigure}
        \vspace{3mm}
\caption{Taxi Environments: (See Fig. 7) The agent (taxi) is allowed to take 2000 steps in an episode ($r = -1$ for each step), and gets a reward $r = +2500$ for correct pick up and drop. ERPO outperforms other algorithms with PPO-B and PPO-DR models showing good results.}
\vspace{2mm}
\label{fig:taxi_results}
\end{figure*}

\mypara{Comparison with models trained over pre-trained (base) models:}
In the distribution-shifted environment, the models trained over baselines perform slightly better than their counterparts trained from scratch. In the Taxi environment, PPO-B shosw relatively close performance to ERPO. The large cliff area and long episodes in the CliffWalking environment prove challenging, with most algorithms failing to converge. However, A2C-B performs nearly as well as ERPO, though it takes slightly longer to achieve convergence. The FrozenLake environment shows that PPO-B and A2C-B perform competently.

\mypara{Benchmarking against Heuristic Search Methods:}
In addition to comparing learning algorithms, we benchmark baseline heuristic search algorithms, specifically $A^*$ and $IDA^*$, in the FrozenLake and CliffWalking environments. These comparisons use Manhattan distance as the heuristic, with other costs aligned to the previously described reward structures. It is important to note that these comparisons serve as baselines, not direct competitors, as the information available to these algorithms differs, making a like-for-like comparison unfair.

In the FrozenLake environment, $A^*$ performs poorly due to its inability to account for proximity to holes, resulting in significant penalties. The cost incurred by $A^*$ ranges from -$5K$ to -$15K$ across different versions of FrozenLake. Conversely, $IDA^*$ performs well on Level 1 with a reward of approximately $2K$, but declines to -$5K$ on Level 3.

In the CliffWalking environment, $IDA^*$ consistently achieves rewards of approximately $2K$ across all levels, though this success requires the starting node to be positioned near the cliff's edge. In contrast, $A^*$ struggles, with rewards ranging from -200 to -300.

While vanilla heuristic search methods may underperform in stochastic environments, methods like stochastic $A^*$ might be more suitable. Additionally, extending these methods to continuous spaces via function approximators is challenging and requires an admissible heuristic. A key difference is that ERPO and other learning-based methods generate policies that can generalize to any start location within the grid, whereas heuristic-based approaches may need to restart the search when encountering previously unexplored states.

\mypara{Preliminary Results on custom Environments:}
In addition to the results on standard gym environments, we also conducted experiments in custom gym environments with a similar reach-avoid mission. The agents in these environments are assigned randomly to one of many pre-determined start locations, and must reach one of the goal locations whilst avoiding obstacles. 
One sample result for a 100x100 grid world with ~20\% new obstacles is as follows:
\vspace{-5pt}
\begin{table}[h!]
    \centering
    \begin{tabular}{lcccc}
        \toprule
        Algorithm & A* & Q-learning & PPO & ERPO \\
        \midrule
        Path Length & 118.41 & 118.00 & 122.43 & 116.05 \\
        \bottomrule
    \end{tabular}
    \caption{Comparison of Algorithm Performance}
    \label{tab:algorithm_performance}
\end{table}
\vspace{-3pt}
We leave the extension to larger custom grids with varying levels of obstacle density for future work.

\mypara{Analysis:} 
Our observations indicate a distinct advantage of ERPO over traditional reinforcement learning algorithms even when trained over the pre-trained models, and domain randomization methods. While PPO, PPO-DR, and A2C utilize batch-wise updates, and DQN depends on episode-wise updates, these algorithms generally treat each step within a batch or episode as equally significant for the purpose of policy updates. This approach can dilute the impact of particularly successful or unsuccessful trajectories on the overall learning process.
In contrast, ERPO places emphasis on trajectories that significantly deviate from the norm — either by outperforming or underperforming compared to the rest of the batch and prioritizes learning from those that are the most informative. This selective update mechanism ensures that ERPO rapidly identifies and leverages the most effective strategies.
As a result of this approach, the fittest trajectories become increasingly predominant in the batch over just a few training episodes. Thus, ERPO leads to a faster and more efficient convergence towards optimal policies.